\documentclass{article}

\usepackage[final, nonatbib]{neurips_2020}

\usepackage[utf8]{inputenc} 
\usepackage[T1]{fontenc}    
\usepackage{hyperref}       
\usepackage{url}            
\usepackage{booktabs}       
\usepackage{amsfonts}       
\usepackage{nicefrac}       
\usepackage{microtype}      

\usepackage{graphicx}
\usepackage{amsthm}
\usepackage{amsmath}
\usepackage{subfig}

\usepackage{array}
\usepackage{amsmath}
\usepackage{graphicx}
\usepackage{makecell}
\usepackage{multicol, blindtext}
\usepackage{soul,color,xcolor}
\usepackage{multirow}
\usepackage[english]{babel}
\usepackage{amssymb}
\usepackage{tabularx}
\usepackage{bm}
\usepackage{comment}
\usepackage{float}
\usepackage{algorithm}
\usepackage{caption}
\usepackage{subfig}
\usepackage{wrapfig}
\usepackage{cite}
\usepackage{algorithm}
\usepackage[noend]{algpseudocode}
\usepackage{adjustbox}

\newtheorem{theorem}{Theorem}
\newtheorem{lemma}{Lemma}
\newtheorem{assumption}{Assumption}
\newtheorem{definition}{Definition}

\title{Estimating the Effects of Continuous-valued Interventions using Generative Adversarial Networks}

\author{Ioana Bica\thanks{Equal contribution.} \\
        Department of Engineering Science\\
        University of Oxford, Oxford, UK\\
        The Alan Turing Institute, London, UK \\
        \texttt{ioana.bica@eng.ox.ac.uk} \\
        \And
        James Jordon$^*$ \\
        Department of Engineering Science\\
        University of Oxford, Oxford, UK\\
        \texttt{james.jordon@wolfson.ox.ac.uk} \\
        \And
        Mihaela van der Schaar \\
        University of Cambridge, Cambridge, UK\\
        University of California, Los Angeles, USA \\
        The Alan Turing Institute, London, UK \\
        \texttt{mv472@cam.ac.uk} \\
}

\begin{document}

\maketitle

\begin{abstract}
 While much attention has been given to the problem of estimating the effect of discrete interventions from observational data, relatively little work has been done in the setting of continuous-valued interventions, such as treatments associated with a dosage parameter. In this paper, we tackle this problem by building on a modification of the generative adversarial networks (GANs) framework. Our model, SCIGAN, is flexible and capable of simultaneously estimating counterfactual outcomes for several different continuous interventions. The key idea is to use a significantly modified GAN model to learn to generate counterfactual outcomes, which can then be used to learn an inference model, using standard supervised methods, capable of estimating these counterfactuals for a new sample.  To address the challenges presented by shifting to continuous interventions, we propose a novel architecture for our discriminator - we build a hierarchical discriminator that leverages the structure of the continuous intervention setting. Moreover, we provide theoretical results to support our use of the GAN framework and of the hierarchical discriminator. In the experiments section, we introduce a new semi-synthetic data simulation for use in the continuous intervention setting and demonstrate improvements over the existing benchmark models.
\end{abstract}

\section{Introduction} \label{sec:intro}
Estimating the personalised effects of interventions is crucial for decision making in many domains such as medicine, education, public policy and advertising. Such domains have a wealth of observational data available. Most of the methods developed in the causal inference literature focus on learning the counterfactual outcomes of discrete interventions, such as binary or categorical treatments\footnote{For ease of exposition, we will sometimes refer to interventions as treatments and to the associated continuous parameter as the dosage throughout the paper.} \cite{bertsimas2017personalized, alaa2017deep, alaa2017bayesian, athey2016recursive, wager2018estimation, yoon2018ganite, alaa2018limits, zhang2020learning, bica2020real}. Unfortunately, in many cases, deciding how to intervene involves not only deciding which intervention to make (e.g. whether to treat cancer with radiotherapy, chemotherapy or surgery) but also deciding on the value of some continuous parameter associated with intervening (e.g. the dosage of radiotherapy to be administered). In medicine there are many examples of treatments that are associated with a continuous dosage parameter (such as vasopressors \cite{dopp2013high}). In the medical setting, using a high dosage for a treatment can lead to toxic effects while using a low dosage can result in no effect on the patient outcome {\cite{wang2017cardiac, zame2020machine}}. In other domains, there are many examples of continuous interventions, such as the duration of an education or job training program, the frequency of an exposure or the price used in an advert. Naturally, being able to estimate the effect of these continuous interventions will aid in the decision making process.

Learning from observational data already presents significant challenges when there is only a single intervention (and thus the decision is binary - whether to intervene or not). As explained in \cite{ite_tutorial}, in an observational dataset, only the factual outcome is present - the ``counterfactual" outcomes are not observed. This problem is exacerbated in the setting of continuous interventions where the number of counterfactuals is no longer even finite. Moreover, the decision to intervene is non-random and instead is assigned according to the features associated with each sample. Due to the continuous nature of the interventions, adjusting for selection bias is significantly more complex than for binary (or even multiple) interventions. Thus, standard methods for adjusting for selection bias for discrete treatments cannot be easily extended to handle bias in the continuous setting.

\begin{wrapfigure}{t}{0.55\textwidth}
    \vspace{-2mm}
    \centering\
    \includegraphics[width=0.55\textwidth]{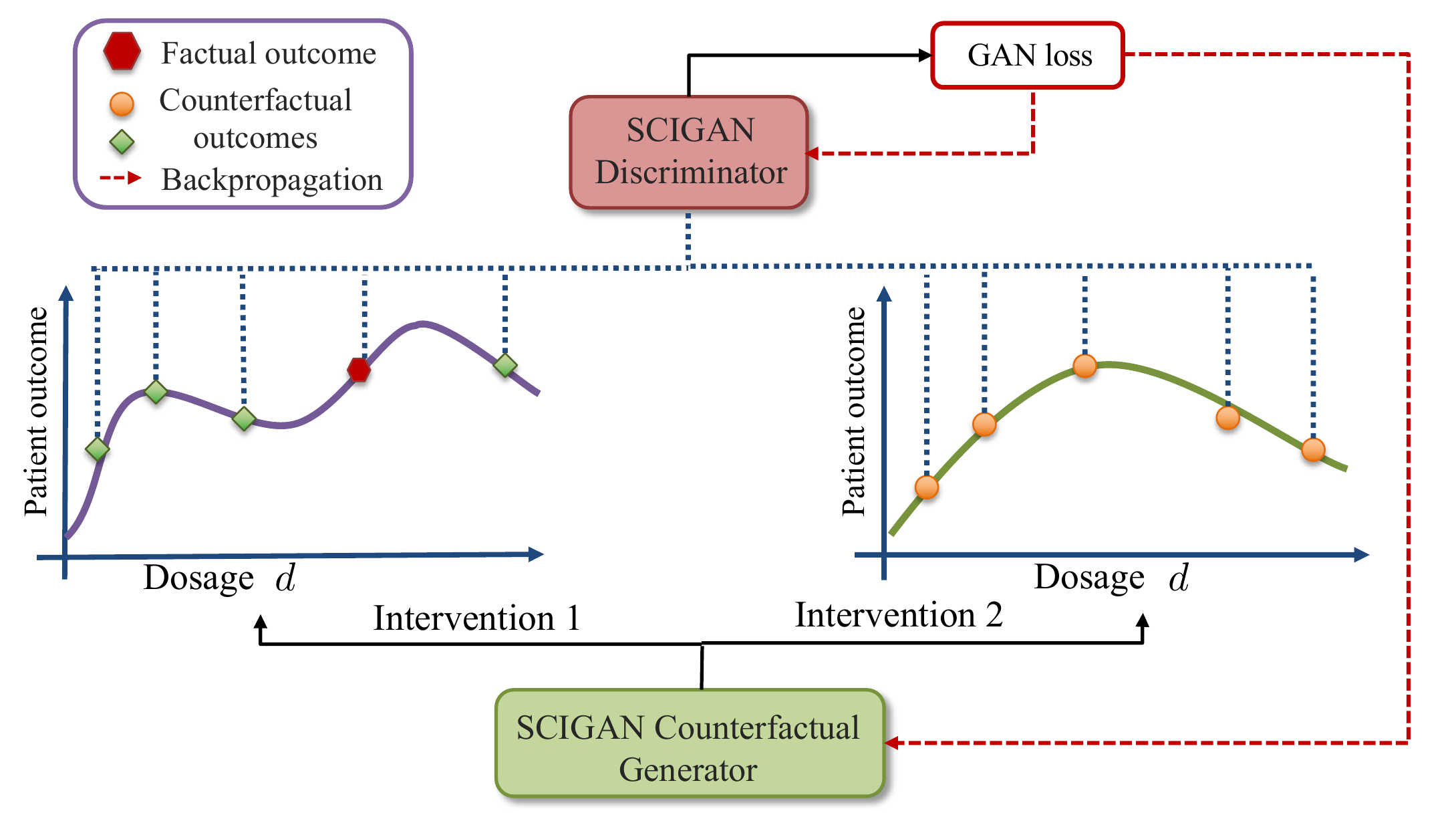}
    \caption{Overview of GAN framework used as part of SCIGAN  for learning the distribution of the counterfactual outcomes.  }
    \label{fig:curves}
    \vspace{-4mm}
\end{wrapfigure}

We propose SCIGAN (eStimating the effects of Continuous Interventions using GANs). We build on the GAN framework of \cite{goodfellow2014generative} to learn the distribution of the unobserved counterfactuals. GANs have already been used in GANITE \cite{yoon2018ganite} to generate the unobserved counterfactual outcomes for discrete interventions. The intuition is that if a counterfactual generator and discriminator are trained adversarially, then the generator can fool the discriminator (i.e. the discriminator will not be able to correctly identify the factual outcome) by generating counterfactuals according to their true distribution. Unfortunately, no theoretical work was provided in \cite{yoon2018ganite} to back up this intuition. A key contribution of this paper is to provide theoretical results that justify using the GAN framework to learn to generate counterfactual outcomes; these results also apply to GANITE.

GANITE itself presents a significant modification to the original GAN framework - rather than the discriminator discriminating between entirely real or entirely fake samples, the discriminator is attempting to identify the real component from a vector containing the real (factual) outcome from the dataset and the fake (counterfactual) outcomes generated by the generator. SCIGAN inherits this key difference from a standard GAN. However, beyond our theoretical contribution, we propose significant changes to the generator and discriminator in order to tackle the more complex problem of estimating outcomes of continuous interventions.

We define a discriminator that acts on a finite set of points from each generated response curve (rather than on entire curves), as shown in Fig. \ref{fig:curves}. We draw on ideas from \cite{zaheer2017deepsets} to ensure that our discriminator acts as a function of a set, rather than of a vector. Moreover, we focus on the setting of interventions {\em each} with an associated continuous parameter. In this setting, we propose a {\em hierarchical} discriminator which breaks down the job of the discriminator into determining the factual intervention and determining the factual parameter using separate networks. We show in the experiments section that this approach significantly improves performance and is more stable than using a single network discriminator. We also model the generator as a multi-task deep network capable of taking a continuous parameter as an input; this gives us the flexibility to learn heterogeneous response curves for the different interventions.

Our contributions in this paper are 4-fold: (1) we propose SCIGAN, a significantly modified GAN framework, capable of estimating outcomes for continuous and many-level-discrete interventions, (2) we provide theoretical justification for both the use of a GAN framework and a hierarchical discriminator, (3) we propose novel architectures for each of our networks, (4) we propose a new semi-synthetic data simulation for use in the continuous intervention setting. We show, using semi-synthetic experiments, that our model outperforms existing benchmarks.

\section{Related work} \label{sec:relw}
Methods for estimating the outcomes of treatments with a continuous dosage from observational data make use of the generalized propensity score (GPS) \cite{imbens2000role, imai2004causal, hirano2004propensity} or build on top of balancing methods for multiple treatments. \cite{schwab2019learning} developed a neural network based method to estimate counterfactuals for multiple treatments and continuous dosages. The proposed Dose Response networks (DRNets) in \cite{schwab2019learning} consist of a three level architecture with shared layers for all treatments, multi-task layers for each treatment and additional multi-task layers for dosage sub-intervals. Specifically, for each treatment $w$, the dosage interval $[a_w, b_w]$ is subdivided into $E$ \textit{equally} sized sub-intervals and a multi-task head is added for each sub-interval. This is an extension of the architecture in \cite{shalit2017estimating}. However, the main advantage of using multi-task heads for dosage intervals would be the added flexibility in the model to learn potentially very different functions over different regions of the dosage interval. DRNets do not determine these intervals dynamically and thus much of this flexbility is lost. Our approach (using GANs to generate counterfactuals) fundamentally differs from DRNets (supervised learning with bias-adjustment) and we demonstrate experimentally that SCIGAN outperforms GPS and DRNets.

For a discussion of works that address treatment-response estimation without a dosage parameter, see Appendix \ref{app:exprelwork}. Note that for such methods we cannot treat the dosage as an input due to the bias associated with its assignment. In Appendix A, we also describe the relationships between our work (causal inference for continuous interventions) and policy optimization with continuous treatments.

\section{Problem formulation} \label{sec:prob}
We consider receiving observations of the form $(\mathbf{x}^i, t_f^i, y_f^i)$ for $i = 1, ..., N$, where, for each $i$, these are independent realizations of the random variables $(\mathbf{X}, T_f, Y_f)$. We refer to $\mathbf{X}$ as the feature vector lying in some feature space $\mathcal{X}$, containing pre-treatment covariates (such as age, weight and lab test results). The treatment random variable, $T_f$, is in fact a pair of values $T_f = (W_f, D_f)$ where $W_f \in \mathcal{W}$ corresponds to the {\em type} of treatment being administered (e.g. chemotherapy or radiotherapy) which lies in the discrete space of $k$ treatments, $\mathcal{W} = \{w_1, ..., w_k\}$, and $D_f$ corresponds to the {\em dosage} of the treatment (e.g. number of cycles, amount of chemotherapy, intensity of radiotherapy), which, for a given treatment $w$ lies in the corresponding treatment's dosage space, $\mathcal{D}_w$ (e.g. the interval $[0, 1]$). We define the set of all treatment-dosage pairs to be $\mathcal{T} = \{(w, d) : w \in \mathcal{W}, d \in \mathcal{D}_w\}$.

Following Rubin's potential outcome framework \cite{rubin1984bayesianly}, we assume that for all treatment-dosage pairs, $(w, d)$, there is a potential outcome $Y(w, d) \in \mathcal{Y}$ (e.g. 1-year survival probability). The {\em observed} outcome is then defined to be $Y_f = Y(W_f, D_f)$. We will refer to the unobserved (potential) outcomes as {\em counterfactuals}.

The goal is to derive {\em unbiased} estimates of the potential outcomes for a given set of input covariates:
\begin{equation}
    \mu(t, \mathbf{x}) = \mathbb{E}[Y(t) | \mathbf{X} = \mathbf{x}]
\end{equation}
for each $t \in \mathcal{T}$, $\mathbf{x} \in \mathcal{X}$. We refer to $\mu(\cdot)$ as the individualised dose-response function. A table summarising our notation is given in Appendix \ref{app:notation}. In order to ensure that this quantity is equal to $\mathbb{E}[Y | \mathbf{X} = \mathbf{x}, T = t]$ and that the dose-response function is identifiable from the observational data, we require the following two assumptions.

\begin{assumption}(Unconfoundedness)
The treatment assignment, $T_f$, and potential outcomes, $Y(w, d)$, are conditionally independent given the covariates $\mathbf{X}$, i.e.
$\{Y(w, d) | w \in \mathcal{W}, d \in \mathcal{D}_w\} \perp \!\!\! \perp T_f | \mathbf{X}$.
\end{assumption}

\begin{assumption}(Overlap)
$\forall \mathbf{x} \in \mathcal{X}$ such that $p(\mathbf{x}) > 0$, we have $1 > p(t|\mathbf{x}) > 0$ for each  $t \in \mathcal{T}$. 
\end{assumption}

\section{SCIGAN} \label{sec:model}
We propose estimating $\mu$ by first training a generator to generate response curves for each sample {\em within} the training dataset. The learned generator can then be used to train an inference network using standard supervised methods. We build on the idea presented in \cite{yoon2018ganite}, using a modified GAN framework to generate potential outcomes conditional on the observed features, treatment and factual outcome. Several changes must be made to both the generator and discriminator architectures and learning paradigms in order to produce a model capable of handling the dose-response setting.

\subsection{Counterfactual Generator}
Our generator, $\mathbf{G} : \mathcal{X} \times \mathcal{T} \times \mathcal{Y} \times \mathcal{Z} \to \mathcal{Y}^{\mathcal{T}}$ takes features, $\mathbf{x} \in \mathcal{X}$, factual outcome, $y_f \in \mathcal{Y}$, received treatment and dosage, $t_f = (w_f, d_f) \in \mathcal{T}$, and some noise, $\mathbf{z} \in \mathcal{Z}$ (typically multivariate uniform or Gaussian), as inputs. The output will be a dose-response curve for each treatment (as shown in Fig. \ref{fig:curves}), so that the output is a function from $\mathcal{T}$ to $\mathcal{Y}$, i.e. $\mathbf{G}(\mathbf{x}, t_f, y_f, \mathbf{z})(\cdot) : \mathcal{T} \to \mathcal{Y}$. We can then write \begin{equation}
    \hat{y}_{cf}(t) = \mathbf{G}(\mathbf{x}, t_f, y_f, \mathbf{z})(t)
\end{equation}
as our generated counterfactual outcome for treatment-dosage pair $t$. We write $\hat{Y}_{cf}(t) = \mathbf{G}(\mathbf{X}, T_f, Y_f, \mathbf{Z})(t)$ (i.e. the random variable induced by $\mathbf{G}$). 

While the job of the counterfactual generator is to generate outcomes for the treatment-dosage pairs which were {\em not} observed, \cite{yoon2018ganite} demonstrated that the performance of the counterfactual generator is improved by adding a supervised loss term that regularises its output for the factual treatment (in our case treatment-dosage pair). We define the supervised loss, $\mathcal{L}_S$, to be
\begin{equation} \label{eq:sup}
    \mathcal{L}_S(\mathbf{G}) = \mathbb{E} \bigg[(Y_f - \mathbf{G}(\mathbf{X}, T_f, Y_f, \mathbf{Z})(T_f))^2\bigg]\,,
\end{equation}
where the expectation is taken over $\mathbf{X}, T_f, Y_f$ and $\mathbf{Z}$.

\vspace{-0.1cm}
\subsection{Counterfactual Discriminator}
\vspace{-0.1cm}

As noted in Section \ref{sec:intro}, our discriminator will act on a random set of points from each of the generated dose-response curves. Similar to \cite{yoon2018ganite}, we define a discriminator, $\mathbf{D}$, that will attempt to pick out the factual treatment-dosage pair from among the (random set of) generated ones. To handle the complexity arising from multiple treatments, we break down our discriminator into two distinct models: a treatment discriminator and a dosage discriminator.

Formally, let $n_w \in \mathbb{Z}^+$ be the number of dosage levels we will compare for treatment $w \in \mathcal{W}$\footnote{In practice we set all $n_w$ to be the same. The default setting is $5$ in the experiments.}.
For each $w \in \mathcal{W}$, let $\tilde{\mathcal{D}}_w = \{D^w_1, ..., D^w_{n_w}\}$ be a random subset\footnote{In practice, when $\mathcal{D}_w = [0, 1]$, each $D^w_j$ is sampled independently and uniformly from $[0, 1]$. Note that for each training iteration, $\tilde{\mathcal{D}}_w$ is resampled (see Section \ref{sec:intro}).} of $\mathcal{D}_w$ of size $n_w$, where for the factual treatment, $W_f$, $\tilde{\mathcal{D}}_{W_f}$ contains $n_{W_f} - 1$ random elements along with $D_f$. We define $\tilde{\mathbf{Y}}_w = (D^w_i, \tilde{Y}^w_i)_{i=1}^{n_w} \in (\mathcal{D}_w \times \mathcal{Y})^{n_w}$ to be the vector of dosage-outcome pairs for treatment $w$ where
\begin{equation} \label{eq:tildey}
    \tilde{Y}^w_i = \begin{cases} Y_f \text{ if } W_f = w \text{ and } D_f = D^w_i \\
    \hat{Y}_{cf}(w, D^w_i) \text{ else}
    \end{cases}
\end{equation}
and $\tilde{\mathbf{Y}} = (\tilde{\mathbf{Y}}_w)_{w \in \mathcal{W}}$. We will write $d^w_j, \tilde{\mathbf{y}}_w$ and $\tilde{\mathbf{y}}$ to denote realisations of $D^w_j, \tilde{\mathbf{Y}}_w$ and $\tilde{\mathbf{Y}}$.

\begin{figure*}[t]
    \centering
    \includegraphics[width=\columnwidth]{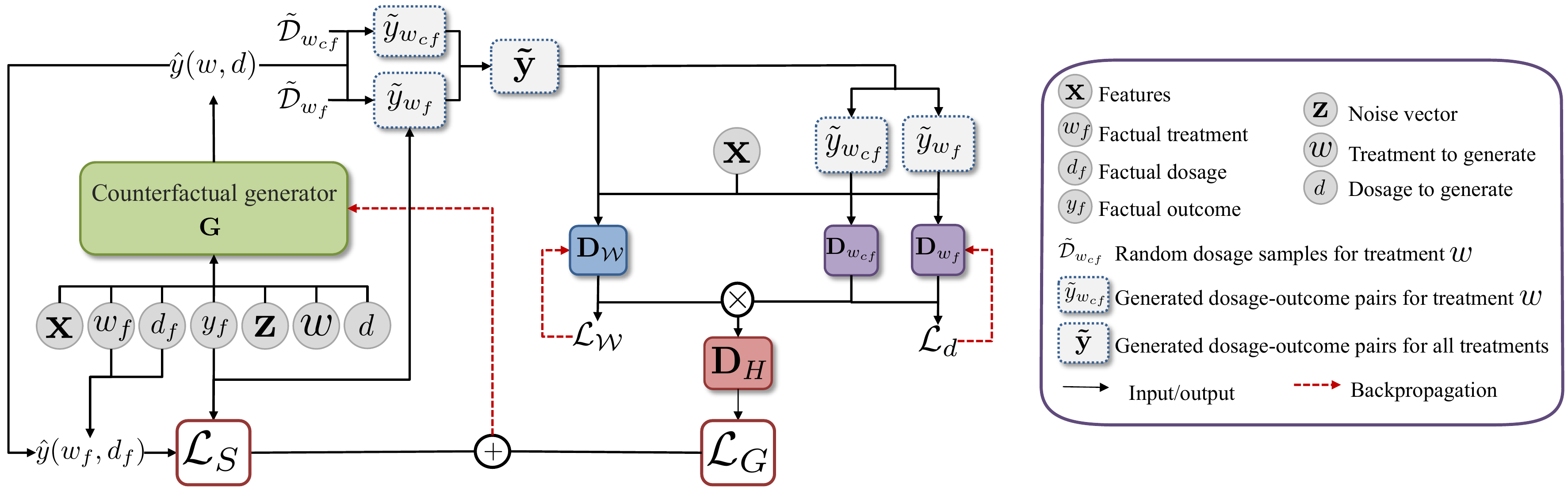}
    \vspace{-0.1cm}
    \caption{Overview of our model for the setting with two treatments ($w_f$/$w_{cf}$ being the factual/counterfactual treatment). The generator is used to generate an output for each dosage level in each $\tilde{\mathcal{D}}_w$, these outcomes together with the factual outcome, $y_f$, are used to create the set of dosage-outcome pairs, $\tilde{\mathbf{y}}$, which is passed to the treatment discriminator. Each dosage discriminator receives only the part of $\tilde{\mathbf{y}}$ corresponding to that treatment, i.e. $\tilde{\mathbf{y}}_w$. These discriminators are combined (Eq. \ref{eq:hierdcomb}) to define $\mathbf{D}_H$ which gives feedback to the generator. }
    \label{fig:comp}
    \vspace{-0.5cm}
\end{figure*}

The treatment discriminator, $\mathbf{D}_\mathcal{W} : \mathbf{X} \times \prod_{w \in \mathcal{W}} (\mathcal{D}_w \times \mathcal{Y})^{n_w} \to [0, 1]^k$, takes the features, $\mathbf{x}$, and generated potential outcomes, $\tilde{\mathbf{y}}$, and outputs a probability for each treatment, $w_1, ..., w_k$. Writing $\mathbf{D}^w_\mathcal{W}$ to denote the output of $\mathbf{D}_\mathcal{W}$ corresponding to treatment $w$, we define the loss, $\mathcal{L}_\mathcal{W}$, to be
\begin{small}
\begin{equation}
    \mathcal{L}_\mathcal{W}(\mathbf{D}_\mathcal{W}; \mathbf{G}) = - \mathbb{E} \Bigg[\sum_{w \in \mathcal{W}} \mathbb{I}_{\{W_f = w\}} \log \mathbf{D}^w_\mathcal{W}(\mathbf{X}, \tilde{\mathbf{Y}}) + \mathbb{I}_{\{W_f \neq w\}} \log(1 - \mathbf{D}^w_\mathcal{W} (\mathbf{X}, \tilde{\mathbf{Y}})) \Bigg]\,,
\end{equation}
\end{small}
where the expectation is taken over $\mathbf{X}, W_f, D_f, \tilde{\mathbf{Y}}$ and $\{\tilde{\mathcal{D}}_w\}_{w \in \mathcal{W}}$. 

Then, for each $w \in \mathcal{W}$, the dosage discriminator, $\mathbf{D}_w: \mathcal{X} \times (\mathcal{D}_w \times \mathcal{Y})^{n_w} \to [0, 1]^{n_w}$, is a map that takes the features, $\mathbf{x}$, and generated potential outcomes, $\tilde{\mathbf{y}}_w$, corresponding to treatment $w$ and outputs a probability for each dosage level, $d^w_1, ..., d^w_{n_w}$, in a given realisation of $\tilde{\mathcal{D}}_w$. Writing $\mathbf{D}_w^j$ to denote the output of $\mathbf{D}_w$ corresponding to dosage level $D^w_j$, we define the loss of each dosage discriminator to be
\begin{small}
\begin{equation}
    \mathcal{L}_{d}(\mathbf{D}_w; \mathbf{G}) = - \mathbb{E} \Bigg[ \mathbb{I}_{\{W_f = w\}} \sum_{j = 1}^{n_w} \mathbb{I}_{\{D_f = D^w_j\}} \log \mathbf{D}_w^j(\mathbf{X}, \tilde{\mathbf{Y}}_w) + \mathbb{I}_{\{D_f \neq D^w_j\}} \log (1 - \mathbf{D}_w^j(\mathbf{X}, \tilde{\mathbf{Y}}_w)) \Bigg] \,,
\end{equation}
\end{small}
where the expectation is taken over $\mathbf{X}, \tilde{\mathcal{D}}_w, \tilde{\mathbf{Y}}_w, W_f$ and $D_f$. The $\mathbb{I}_{\{W_f = w\}}$ term ensures that only samples for which the factual treatment is $w$ are used to train dosage discriminator $\mathbf{D}_w$ (otherwise there would be no factual dosage for that sample).

We define the overall discriminator $\mathbf{D}_H : \mathcal{X} \times \prod_{w \in \mathcal{W}} (\mathcal{D}_w \times Y)^{n_w} \to [0, 1]^{\sum n_w}$ 
by defining its output corresponding to the treatment-dosage pair $(w, d^w_j)$ as
\begin{equation} \label{eq:hierdcomb}
    \mathbf{D}_H^{w, j}(\mathbf{x}, \tilde{\mathbf{y}}) = \mathbf{D}^w_\mathcal{W}(\mathbf{x}, \tilde{\mathbf{y}}) \times \mathbf{D}_w^j(\mathbf{x}, \tilde{\mathbf{y}}_w)\,.
\end{equation}

Instead of the standard GAN minimax game, the generator and discriminators are trained according to the minimax game defined by seeking $\mathbf{G}^*$, $\mathbf{D}_H^*$ that solve:
\begin{align} 
    &\mathbf{G}^* = \arg \min_\mathbf{G} \mathcal{L}(\mathbf{D}^*_H; \mathbf{G}) + \lambda \mathcal{L}_S(\mathbf{G}) \quad \quad \label{eq:hierD_H} &&{\mathbf{D}^{*}_H}^{w, j} = {\mathbf{D}^*_\mathcal{W}}^w \times {\mathbf{D}^*_w}^j \\ 
    &\mathbf{D}^*_\mathcal{W} = \arg \min_{\mathbf{D}_\mathcal{W}} \mathcal{L}_\mathcal{W}(\mathbf{D}_\mathcal{W}; \mathbf{G}^*) \quad \quad && \label{eq:hierD_w} \mathbf{D}^*_w = \arg \min_{\mathbf{D}_w} \mathcal{L}_d(\mathbf{D}_w; \mathbf{G}^*), \forall w \in \mathcal{W}
\end{align}

Fig. \ref{fig:comp} depicts our generator and hierarchical discriminator. Pseudo-code for our algorithm can be found in Appendix \ref{app:pseudo}.

\subsection{Inference Network}
Once we have learned the counterfactual generator, we can use it only to access (generated) dose-response curves for all samples in the dataset. To generate dose-response curves for a new sample we use the counterfactual generator along with the original data to train an inference network, $\mathbf{I} : \mathcal{X} \times \mathcal{T} \to \mathcal{Y}$. Details of the loss and pseudo-code can be found in Appendix \ref{app:infnet}.

\subsection{Theoretical Analysis}


We now state our key theorem: the game defined by Eqs. (\ref{eq:hierD_H}-\ref{eq:hierD_w}) results in our hierarchical GAN learning counterfactuals that agree (in marginal distribution) with the true data.


\begin{theorem} \label{thm:main}
The global minimum of $\mathcal{L}(\mathbf{D}^*_H; \mathbf{G}) + \lambda \mathcal{L}_S(\mathbf{G})$ subject to ${\mathbf{D}^{*}_H}^{w, j} = {\mathbf{D}^*_\mathcal{W}}^w \times {\mathbf{D}^*_w}^j$, $\mathbf{D}^*_\mathcal{W} = \arg \min_{\mathbf{D}_\mathcal{W}} \mathcal{L}_\mathcal{W}(\mathbf{D}_\mathcal{W}; \mathbf{G}^*)$ and $\mathbf{D}^*_w = \arg \min_{\mathbf{D}_w} \mathcal{L}_d(\mathbf{D}_w; \mathbf{G}^*), \forall w \in \mathcal{W}$ is achieved if and only if for all $\tilde{\mathcal{D}}_w$, for all $w, w' \in \mathcal{W}$ and for all $d \in \tilde{\mathcal{D}}_w$, $d' \in \tilde{\mathcal{D}}_{w'}$
\begin{equation} \label{eq:marg_eq}
    p_{w, d}(\mathbf{y} | \mathbf{x}) = p_{w', d'}(\mathbf{y} | \mathbf{x})
\end{equation}
which in turn implies that for any $(w, d) \in \mathcal{T}$ we have that the generated counterfactual for outcome $(w, d)$ for any sample (that was not assigned $(w, d)$) has the same (marginal) distribution (conditional on the features) as the true marginal distribution for that outcome.
\end{theorem}
\begin{proof}
The proof and intermediate results can be found in Appendix \ref{app:proofs}.
\end{proof}



\section{Architecture} \label{sec:arch}
In this section, we describe in detail the novel architectures that we adopt to model each of the functions $\mathbf{G}, \mathbf{D}, \mathbf{D}_\mathcal{W}, \mathbf{D}_{w_1}, ..., \mathbf{D}_{w_k}$ which draws from the ideas in \cite{zaheer2017deepsets}. The inference network, $\mathbf{I}$, has the same architecture as the generator, but does not receive $w_f, d_f, y_f$ or $\mathbf{z}$ as inputs.

\subsection{Generator Architecture}
We adopt a multi-task deep learning model for $\mathbf{G}$ by defining a function $g : \mathcal{X} \times \mathcal{T} \times \mathcal{Y} \times \mathcal{Z} \to \mathcal{H}$ for some latent space $\mathcal{H}$ (typically $\mathbb{R}^l$ for some $l$) and then for each treatment $w \in \mathcal{W}$ we introduce a multitask ''head", $g_w : \mathcal{H} \times \mathcal{D}_w \to \mathcal{Y}$ taking inputs from $\mathcal{H}$ {\em and} a dosage, $d$, to produce an outcome $\hat{y}(w, d) \in \mathcal{Y}$. Given observations, $(\mathbf{x}, t_f, y_f)$, a noise vector $\mathbf{z}$, and a target treatment-dosage pair, $t = (w, d)$, we define
\begin{equation}
    \mathbf{G}(\mathbf{x}, t_f, y_f, \mathbf{z})(t) = g_w(g(\mathbf{x}, t_f, y_f, \mathbf{z}), d)\,.
\end{equation}
Each of $g, g_{w_1}, ..., g_{w_k}$ are fully connected networks. A figure of our generator architecture is given in Figure \ref{fig:arch}(a).

\subsection{Hierarchical Discriminator Architectures} \label{sec:hier_arch}


In order to ensure the discriminators act as functions of sets, we use ideas from \cite{zaheer2017deepsets} to create permutation invariant and permutation equivariant\footnote{Definitions can be found in \cite{zaheer2017deepsets} and in Appendix \ref{app:inveqvdef}.} networks. Zaheer et al. \cite{zaheer2017deepsets} provide several possible building blocks to use to construct invariant and equivariant deep networks. The basic building block we will use for invariant functions will be a layer of the form:
\begin{equation} \label{eq:invar}
    f_{inv}(\mathbf{u}) = \sigma(\mathbf{1}_b \mathbf{1}_m^T (\phi(u_1), ..., \phi(u_m)))\,,
\end{equation}
where $\mathbf{1}_l$ is a vector of $1$s of dimension $l$, $\phi$ is any function $\phi : \mathcal{U} \to \mathbb{R}^q$ for some $q$ (in this paper we use a standard fully connected layer) and $\sigma$ is some non-linearity. The basic building block for equivariant functions is defined in terms of equivariance input, $\mathbf{u}$, and auxiliary input, $\mathbf{v}$, by:
\begin{equation} \label{eq:equivar}
    f_{equi}(\mathbf{u}, \mathbf{v}) = \sigma(\lambda \mathbf{I}_m \mathbf{u} + \gamma (\mathbf{1}_m \mathbf{1}_m^T)\mathbf{u} + (\mathbf{1}_m\Theta^T)\mathbf{v})\,,
\end{equation}
where $\mathbf{I}_m$ is the $m \times m$ identity matrix, $\lambda$ and $\gamma$ are scalar parameters and $\Theta$ is a vector of weights.

\begin{figure}
    \centering
    \subfloat[Generator]{\includegraphics[height=5.0cm]{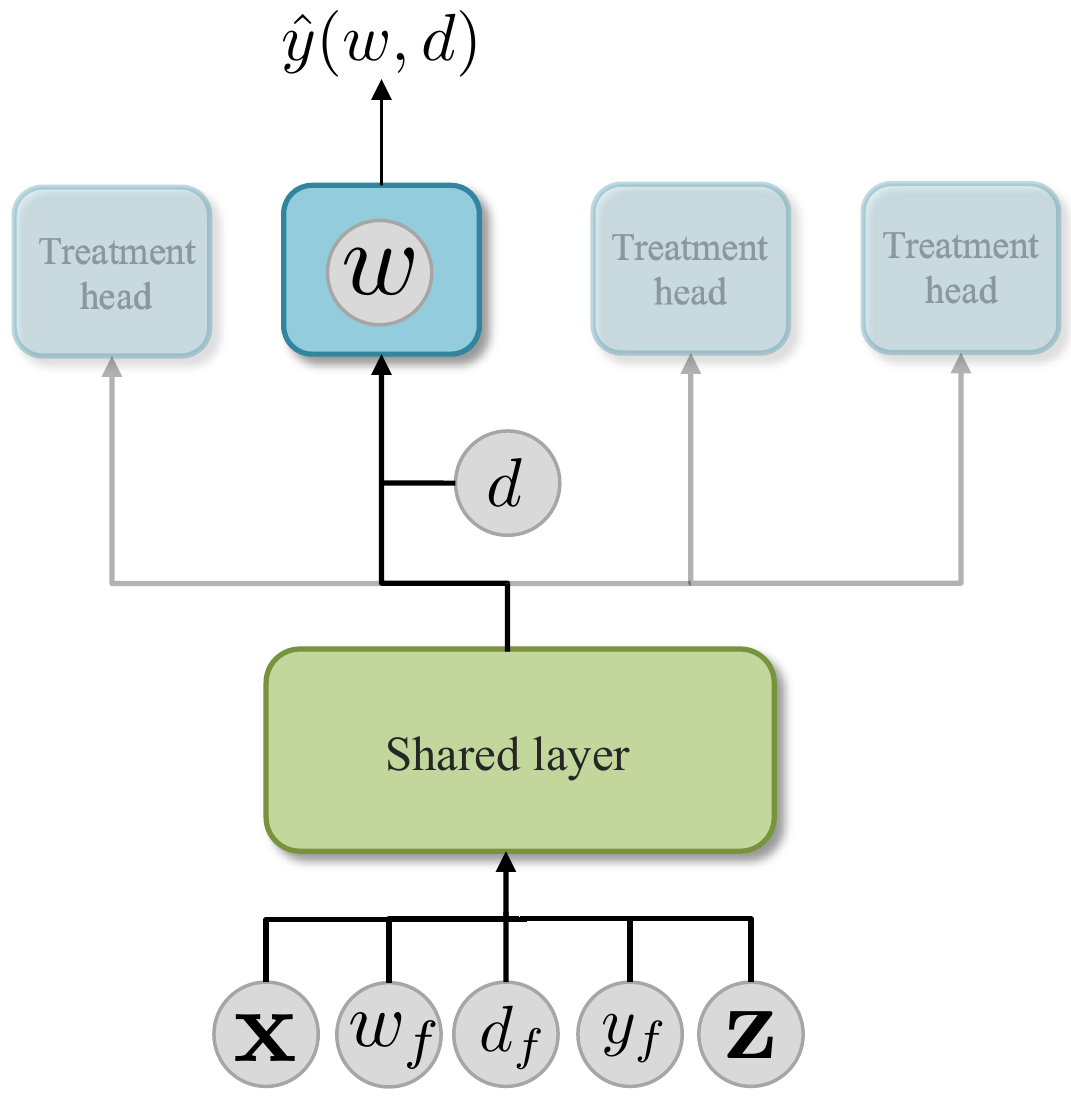}}
    \quad
    \subfloat[Treatment Discriminator]{\includegraphics[height=5.0cm]{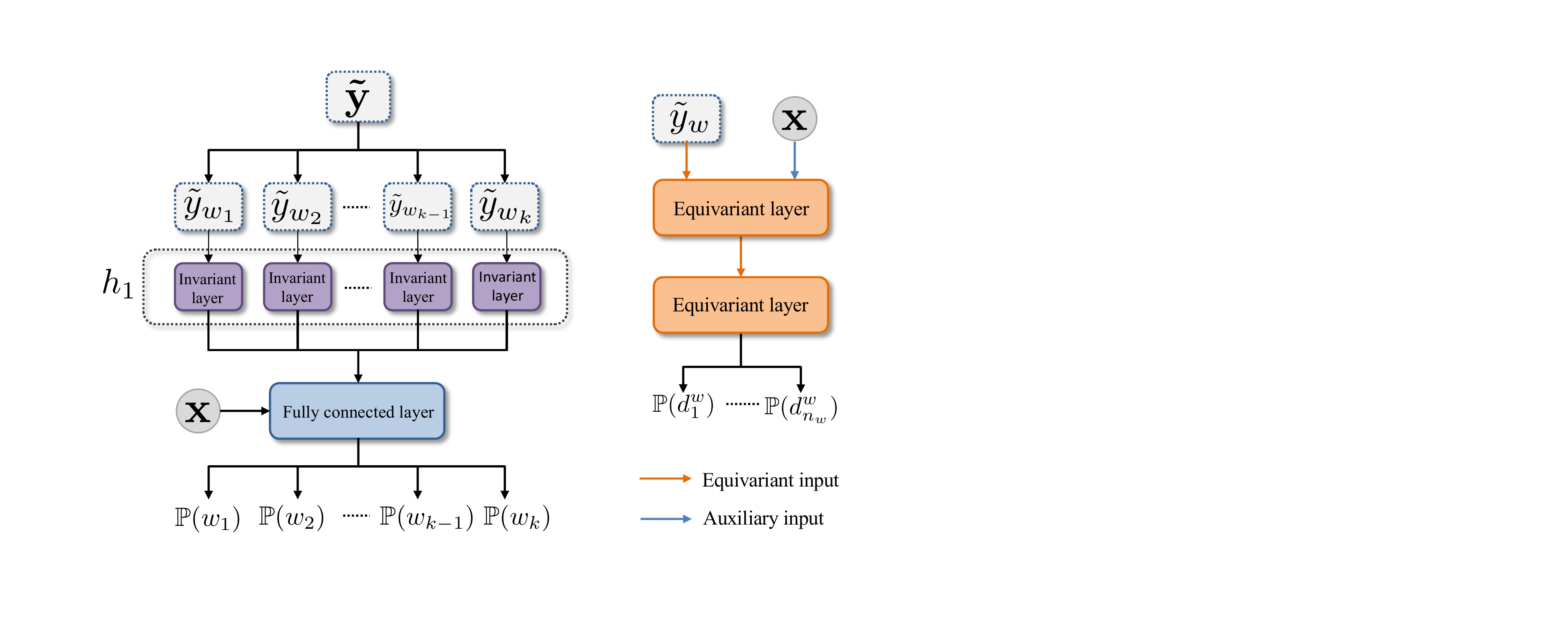} }
    \quad
    \subfloat[Dosage \newline Discriminator]{\includegraphics[height=5.0cm]{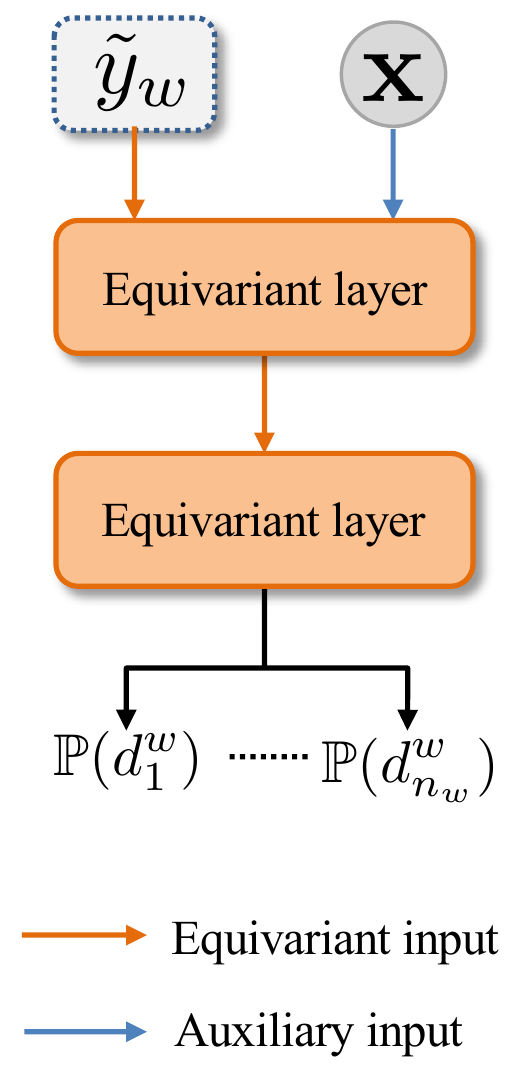}}
    \caption{Architecture of our generator and discriminators.}
    \label{fig:arch}
    \vspace{-3mm}
\end{figure}


In the case of the hierarchical discriminator, we want the treatment discriminator, $\mathbf{D}_\mathcal{W}$, to be permutation invariant with respect to $\tilde{\mathbf{y}}_w$ for each treatment. To achieve this we define $h_1: \prod_{w \in \mathcal{W}} (\mathcal{D}_w \times \mathcal{Y})^{n_w} \to \mathcal{H}_H$ and require that $h_1$ be permutation invariant w.r.t. each of the spaces $(\mathcal{D}_w \times \mathcal{Y})^{n_w}$. We concatenate the output of $h_1$ with the features $\mathbf{x}$ and pass these through a fully connected network $h_2 : \mathcal{X} \times \mathcal{H}_H \to [0, 1]^k$ so that $\mathbf{D}_\mathcal{W}(\mathbf{x}, \tilde{\mathbf{y}}) = h_2(\mathbf{x}, h_1(\tilde{\mathbf{y}}))$.

To construct $h_1$, we concatenate the outputs of several invariant layers of the form given in Eq. (\ref{eq:invar}) that each individually act on the spaces $(\mathcal{D}_w \times \mathcal{Y})^{n_w}$. That is, for each treatment, $w \in \mathcal{W}$ we define a map $h^w_{inv} : (\mathcal{D}_w \times \mathcal{Y})^{n_w} \to \mathcal{H}^w_H$ by substituting $\tilde{\mathbf{y}}_w$ for $\mathbf{u}$ in Eq. (\ref{eq:invar}). We then define $\mathcal{H}_H = \prod_{w \in \mathcal{W}} \mathcal{H}_H^w$ and  $h_1(\tilde{\mathbf{y}}) = (h^{w_1}_{inv}(\tilde{\mathbf{y}}_{w_1}), ..., h^{w_k}_{inv}(\tilde{\mathbf{y}}_{w_k}))$.

We want each dosage discriminator, $\mathbf{D}_w$, to be permutation equivariant with respect to $\tilde{\mathbf{y}}_w$. To achieve this each $\mathbf{D}_w$ will consist of two layers of the form given in Eq. (\ref{eq:equivar}) with the equivariance input, $\mathbf{u}$, to the first layer being $\tilde{\mathbf{y}}_w$ and to the second layer being the output of the first layer and the auxiliary input, $\mathbf{v}$, to the first layer being the features, $\mathbf{x}$, and then no auxiliary input to the second layer. 

Diagrams depicting the architectures of the treatment discriminator and dosage discriminators can be found in Fig. \ref{fig:arch}(b) and Fig. \ref{fig:arch}(c) respectively.

\section{Evaluation}
\vspace{-0.1cm}
The nature of the treatment-effects estimation problem in even the binary treatments setting does not allow for meaningful evaluation on real-world datasets due to the inability to observe the counterfactuals. While there are well-established benchmark synthetic models for use in the binary (or multiple) case, no such models exist for the dosage setting. We propose our own semi-synthetic data simulation to evaluate our model against several benchmarks.

\subsection{Experimental setup}
\textbf{Semi-synthetic data generation:} We simulate data as follows. We obtain features, $\mathbf{x}$, from a real dataset (in this paper we use TCGA \cite{weinstein2013cancer}, News \cite{johansson2016learning, schwab2019learning}) and MIMIC III \cite{mimiciii})\footnote{Details of each dataset can be found in Appendix \ref{app:data}}. We consider 3 treatments each accompanied by a dosage. Each treatment, $w$, is associated with a set of parameters, $\mathbf{v}^w_1, \mathbf{v}^w_2$, $\mathbf{v}^w_3$. For each run of the experiment, these parameters are sampled randomly by sampling a vector, $\mathbf{u}^w_i$, from $\mathcal{N}(\mathbf{0}, \mathbf{1})$ and then setting $\mathbf{v}^w_i = \mathbf{u}^w_i / {||\mathbf{u}^w_i||}$ where $||\cdot||$ is Euclidean norm. The shape of the response curve for each treatment, $f_w(\mathbf{x}, d)$, is given in Table \ref{tab:data_simultation}, along with a closed-form expression for the optimal dosage. We add $\epsilon \sim \mathcal{N}(0, 0.2)$ noise to the outcomes.

We assign interventions by sampling a dosage, $d_w$, for each treatment from a beta distribution, $d_w | \mathbf{x} \sim$ Beta$(\alpha, \beta_w)$. $\alpha \geq 1$ controls the dosage selection bias ($\alpha = 1$ gives the uniform distribution - see Appendix \ref{app:dosass}). $\beta_w = \frac{\alpha - 1}{d^*_w} + 2 - \alpha$, where $d^*_w$ is the optimal dosage\footnote{For symmetry, if $d_w^* = 0$, we sample $d_w$ from $1 - \text{Beta}(\alpha, \beta_w)$ where $\beta_w$ is set as though $d_w^* = 1$.} for treatment $w$. This setting of $\beta_w$ ensures that the mode of Beta$(\alpha, \beta_w)$ is $d_w^*$. We then assign a treatment according to $w_f | \mathbf{x} \sim \text{Categorical}(\text{softmax}(\kappa f(\mathbf{x}, d_w))$ where increasing $\kappa$ increases selection bias, and $\kappa = 0$ leads to random assignments. The factual intervention is given by $(w_f, d_{w_f})$. Unless otherwise specified, we set $\kappa = 2$ and $\alpha = 2$.

\begin{table*}[h!]
    	\centering
  
    	\begin{tabular}{ ccc }
    		\toprule
    		\textbf{Treatment}& Dose-Response & Optimal dosage\\
    		\midrule
    		1 & $f_1(\mathbf{x}, d) = C((\mathbf{v}^1_1)^T \mathbf{x} + 12(\mathbf{v}^1_2)^T \mathbf{x} d - 12(\mathbf{v}^1_3)^T \mathbf{x} d^2$) & $d_1^{*} = \frac{(\mathbf{v}^1_2)^T  \mathbf{x}}{2(\mathbf{v}^1_3)^T \mathbf{x}}$ \\
    		\hline
    		2 & $f_2(\mathbf{x}, d) = C((\mathbf{v}^2_1)^T \mathbf{x} + \sin (\pi (\frac{{\mathbf{v}^2_2}^T \mathbf{x}}{{\mathbf{v}^2_3}^T \mathbf{x}}) d))$ & $d_2^* = \frac{(\mathbf{v}^2_3)^T \mathbf{x}}{2 (\mathbf{v}^2_2)^T \mathbf{x}}$ \\
    		\hline
    		3 &  $f_3(\mathbf{x}, d) = C((\mathbf{v}^3_1)^T \mathbf{x} + 12d(d - b)^2$, where $b = 0.75 \frac{(\mathbf{v}^3_2)^T \mathbf{x}}{(\mathbf{v}^3_3)^T \mathbf{x}})$  & \makecell{$\frac{b}{3}$ if $b \geq 0.75$ \\ 1 if $b < 0.75$}  \\
    		\bottomrule
    	\end{tabular}
    	\caption{Dose response curves used to generate semi-synthetic outcomes for patient features $\mathbf{x}$. In the experiments, we set $C = 10$.  $\mathbf{v}^w_1, \mathbf{v}^w_2$, $\mathbf{v}^w_3$ are the parameters associated with each treatment $w$.}    
    	\label{tab:data_simultation}

    \end{table*}


\textbf{Benchmarks:} We compare against Generalized Propensity Score (GPS) \cite{imbens2000role} and Dose Reponse Networks (DRNet) \cite{schwab2019learning} (the standard model and with Wasserstein regularization (DRN-W)). As a baseline, we compare against a standard multilayer perceptron (MLP) that takes patient features, treatment and dosage as input and estimates the patient outcome and a multitask variant (MLP-M) that has a designated head for each treatment. See Appendix \ref{app:benchmarks} for details of the benchmark models and their hyperparameter optimisation. For metrics, we use Mean Integrated Square Error (MISE), Dosage Policy Error (DPE) and Policy Error (PE) \cite{silva2016observational, schwab2019learning}. For further details, see Appendix \ref{app:metrics}.

\newpage
\subsection{Source of gain}
\begin{wraptable}{r}{0.5\columnwidth}
    \vspace{-0.5cm}
        \centering
        \begin{small}
        \setlength\tabcolsep{3pt}
        \begin{tabular}{lccc}
        \toprule
        & $\sqrt{\text{MISE}}$ & $\sqrt{\text{DPE}}$ & $\sqrt{\text{PE}}$ \\
        \midrule
         Baseline & $4.18 \pm 0.32$ & $2.06 \pm 0.16$ & $1.93\pm 0.12$ \\
        \cmidrule{1-4}
        {+ $\mathcal{L}_S$} & $3.37 \pm 0.11$ & $1.14 \pm 0.05$ & $0.84 \pm 0.05$ \\
        \cmidrule{1-4}
        {+ Multitask} & $3.15 \pm 0.12$ & $0.85 \pm 0.05$ & $0.67 \pm 0.05$ \\
        \cmidrule{1-4}
        {+ Hierchical} & $2.54 \pm 0.05$  & $0.36\pm 0.05$ & $0.45 \pm 0.05$ \\
        \cmidrule{1-4}
        {+ Inv/Eqv} & $1.89 \pm 0.05$ & $0.31 \pm 0.05$ & $0.25 \pm 0.05$ \\
        \bottomrule
        \end{tabular}
        \caption{Source of gain analysis for our model on the TCGA. Metrics are reported as Mean $\pm$ Std.
        }
        \label{tab:source_of_gain}
        \end{small}
    \vspace{-3mm}
\end{wraptable}

Before comparing against the benchmarks, we investigate how each component of our model affects performance. We start with a baseline model in which both the generator and discriminator consist of a single fully connected network. One at a time, we add in the following components (cumulatively until we reach our full model): (1) the supervised loss in Eq. \ref{eq:sup} (+ $\mathcal{L}_S$), (2) multitask heads in the generator (+ Multitask), (3) hierarchical discriminator (+ Hierarchical) and (4) invariance/equivariance layers in the treatment and dosage discriminators (+Inv/Eqv). We report the results in Table \ref{tab:source_of_gain} for TCGA for all 3 error metrics (MISE, DPE and PE), computed over 30 runs (results on News can be found in Appendix \ref{app:addresults}). 

The addition of each component results in improved performance, with the final row (our full model) demonstrating the best performance across both datasets and for all metrics. In Appendix \ref{app:add_hype} we further compare our hierarchical discriminator with a single network discriminator by investigating both models sensitivity to the hyperparameter $n_w$. Details of the single discriminator can be found in Appendix \ref{app:single}. Architectures for other components of the ablation studies can be found in Appendix \ref{app:ablation_arch}.

\vspace{-0.1cm}
\subsection{Benchmarks comparison} \label{sec:bench_results}
\vspace{-0.1cm}
We now compare SCIGAN\footnote{The implementation of SCIGAN can be found at \url{https://bitbucket.org/mvdschaar/mlforhealthlabpub/} and at \url{https://github.com/ioanabica/SCIGAN}.} against the benchmarks on our 3 semi-synthetic datasets. For MIMIC, due to the low number of samples available, we use two treatments - 2 and 3. We report $\sqrt{\text{MISE}}$ and $\sqrt{\text{PE}}$ in Table \ref{tab:benchmarks}, $\sqrt{\text{DPE}}$ is given in Appendix \ref{app:add_benchmark}. We see that SCIGAN demonstrates a statistically significant improvement over every benchmark across all 3 datasets. In Appendix \ref{app:add_num_treat} we compare SCIGAN with DRNET and GPS for an increasing number of treatments.

\begin{table*}[h]
    	\centering
    	\begin{small}
    	\setlength\tabcolsep{4.2pt}
    	\begin{tabular}{lcccccc}
    		\toprule
    		\multirow{2}{*}{\textbf{Method}} & \multicolumn{2}{c}{\textbf{TCGA}} &
    		\multicolumn{2}{c}{\textbf{News}} & \multicolumn{2}{c}{\textbf{MIMIC}} \\
    		& \multicolumn{1}{c}{$\sqrt{\text{MISE}}$} & \multicolumn{1}{c}{$\sqrt{\text{PE}}$} & \multicolumn{1}{c}{$\sqrt{\text{MISE}}$} & \multicolumn{1}{c}{$\sqrt{\text{PE}}$} & \multicolumn{1}{c}{$\sqrt{\text{MISE}}$} & \multicolumn{1}{c}{$\sqrt{\text{PE}}$} \\
    		\midrule
        	SCIGAN & $ \mathbf{1.89 \pm 0.05}$ &  $ \mathbf{0.25 \pm 0.05}$ &  $ \mathbf{3.71 \pm 0.05}$ &  $ \mathbf{3.90 \pm 0.05}$ &  $ \mathbf{2.09 \pm 0.12}$ &  $ \mathbf{0.32 \pm 0.05}$ \\
    		\midrule
    		DRNet & $3.64\pm 0.12 $ & $0.67 \pm 0.05 $ & $4.98 \pm 0.12 $ & $4.17 \pm {0.11} $ & $4.45\pm 0.12 $ & $1.44 \pm 0.05 $ \\
    		DRN-W & $3.71 \pm 0.12 $ & $0.63 \pm 0.05 $ & $5.07\pm 0.12 $ & $4.56\pm 0.12 $ & $4.47 \pm 0.12$ & $1.37 \pm 0.05$ \\
    		GPS & $4.83 \pm 0.01 $ & $1.60 \pm 0.01 $  & $6.97\pm 0.01 $ & $24.1\pm 0.05 $ & $7.39\pm 0.00 $ & $20.2\pm 0.01  $ \\
    		\midrule
    		MLP-M & $3.96 \pm 0.12 $ & $1.20 \pm 0.05 $ & $5.17 \pm 0.12 $ & $5.82 \pm 0.16 $ &$4.97 \pm 0.16$  & $1.59 \pm 0.05$  \\
    		MLP & $4.31 \pm 0.05 $ & $0.97 \pm 0.05 $ & $5.48 \pm 0.16 $ & $6.45 \pm {0.21} $  & $5.34 \pm 0.16 $ & $1.65 \pm 0.05$  \\
    		\bottomrule
    	\end{tabular}
    	\end{small}
    	\caption{Performance of individualized treatment-dose response estimation on three datasets.  Bold indicates the method with the best performance for each dataset. Metrics are reported as Mean $\pm$ Std.}
    	\label{tab:benchmarks}
    \end{table*}
    
\begin{wrapfigure}{r}{0.37\columnwidth}
    \vspace{-0.8cm}
    \centering
	 \includegraphics[width=0.33\columnwidth]{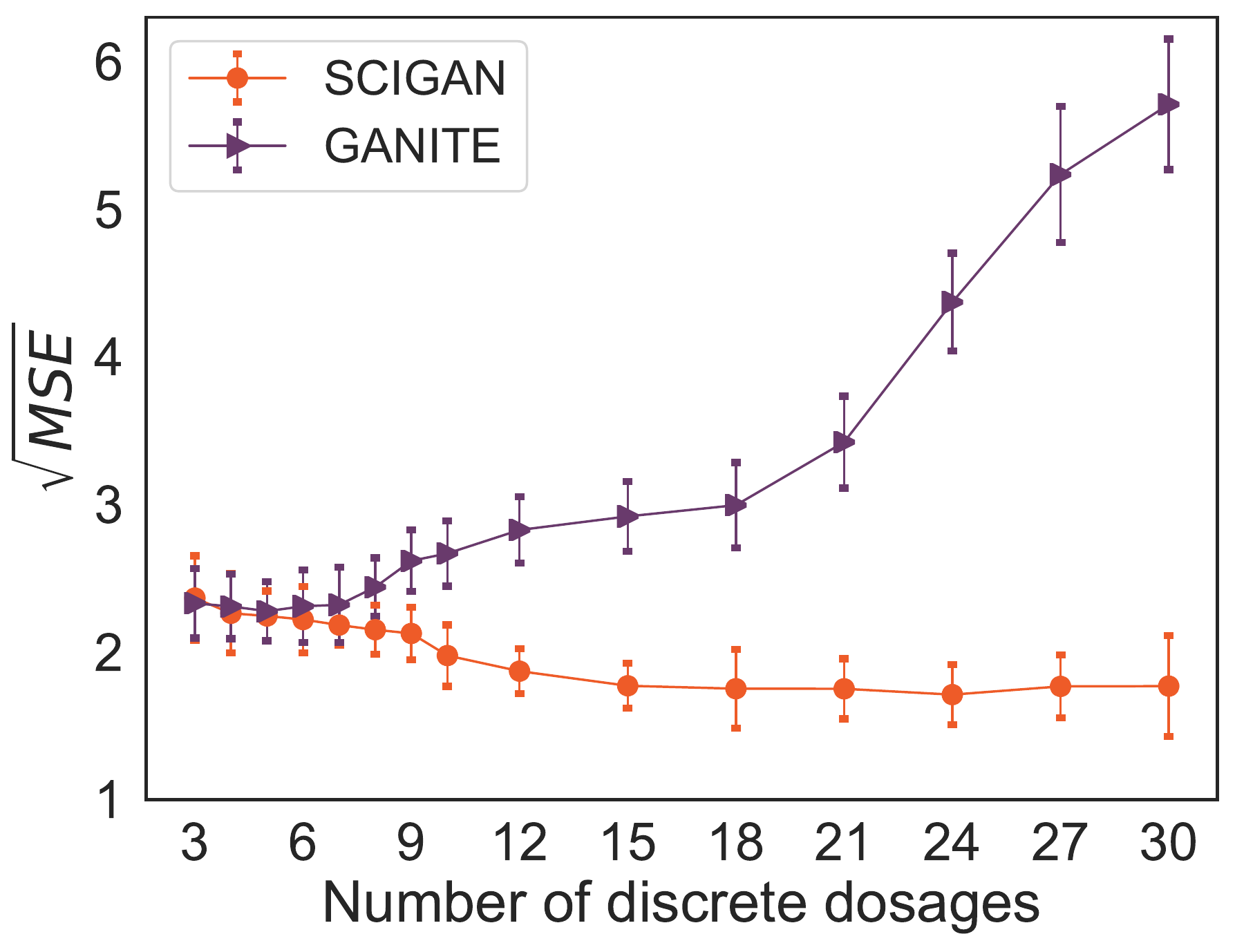}
	 \vspace{-0.3cm}
	\caption{Comparison between SCIGAN and GANITE.}
	\label{fig:discrete_dosage}
	\vspace{-0.6cm}
\end{wrapfigure}

\vspace{-0.1cm}
\subsection{Discrete dosages}
\vspace{-0.1cm}
In this experiment, we investigate the discrete dosage setting. Details of the experimental setup can be found in Appendix \ref{app:disc_dos}. We report Mean Squared Error (MSE) of SCIGAN and GANITE in Fig. \ref{fig:discrete_dosage} where we vary the number of discrete dosages from 3 to 30. We see that GANITE is incapable of handling more than 7 discrete dosages, whereas the hierarchical discriminator together with the invariant and equivariant layers allow SCIGAN to maintain performance as the number of dosages increases. Importantly, this demonstrates SCIGAN's wide-ranging applicability in both discrete and continuous settings.

\newpage
\subsection{Treatment and dosage selection bias} \label{sec:bias}
\vspace{-0.1cm}
Finally, we assess each model's robustness to treatment and dosage bias. We report $\sqrt{\text{MISE}}$ and $\sqrt{\text{PE}}$  on TCGA here. For the other metrics see Appendix \ref{app:add_bias}. Fig. \ref{fig:selection_bias_mise}(a) shows the performance of the 4 methods for $\kappa$ between $0$ (no bias) and $10$ (strong bias). Fig. \ref{fig:selection_bias_mise}(b) shows the performance for $\alpha$ between $1$ (no bias) and $8$ (strong bias). SCIGAN shows consistent performance, significantly outperforming the benchmarks for all $\kappa$ and $\alpha$.
\begin{figure}[H]
    \vspace{-5mm}
    \centering
    \subfloat[Treatment bias]{\includegraphics[width=0.23\linewidth]{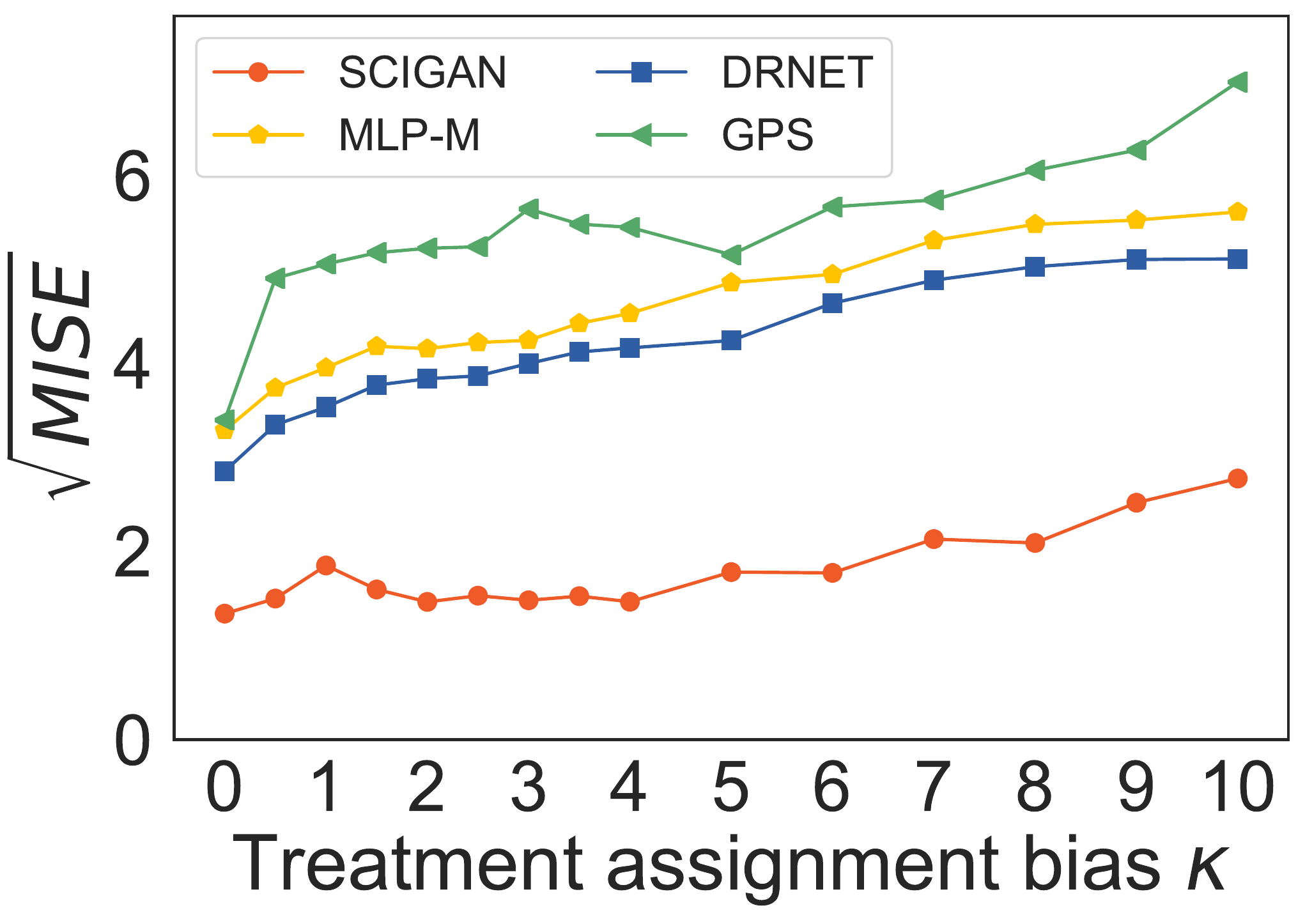}}
    \hfill
    \subfloat[Treatment bias ]{\includegraphics[width=0.23\linewidth]{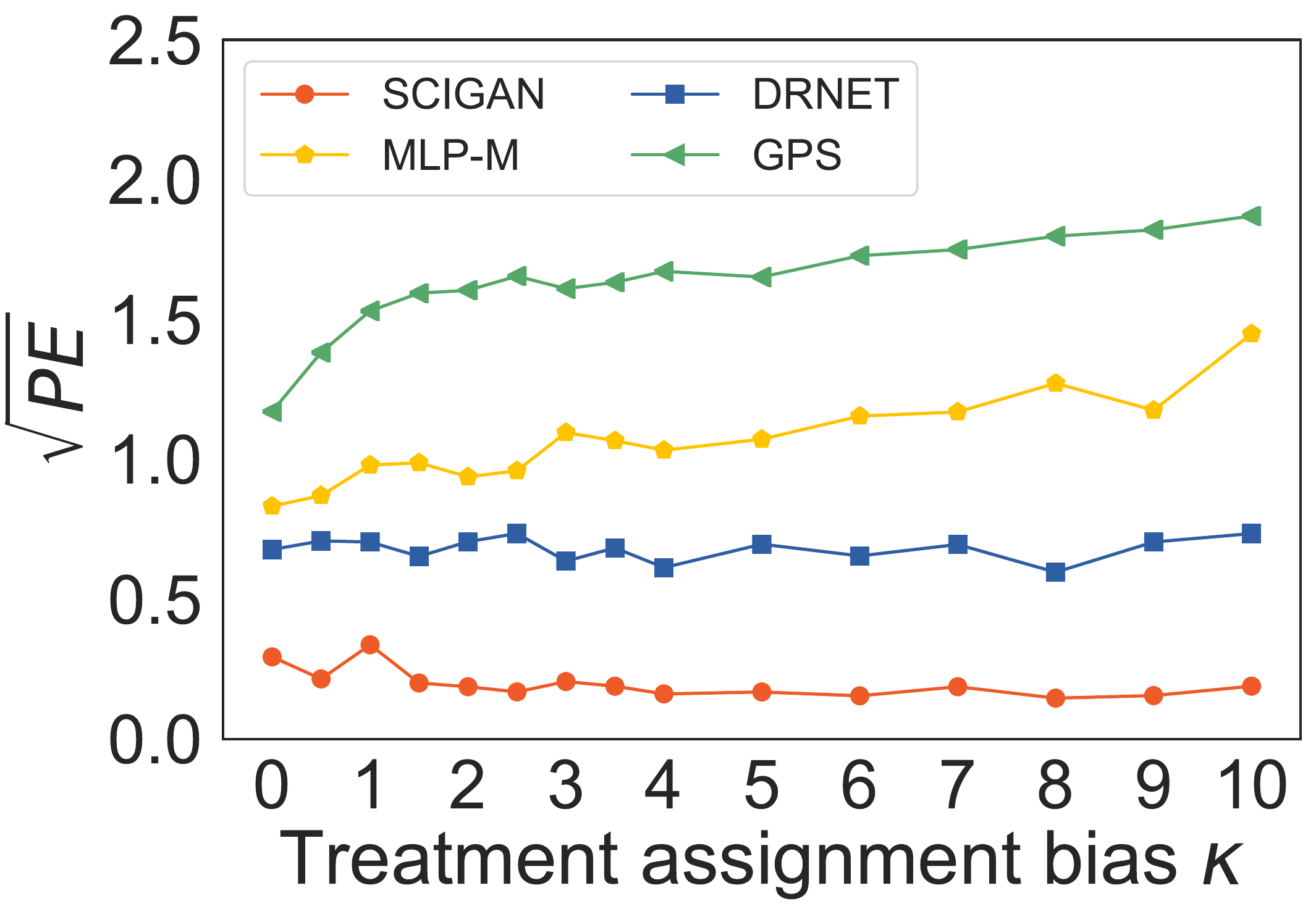}}
    \hfill
     \subfloat[Dosage bias]{\includegraphics[width=0.23\linewidth]{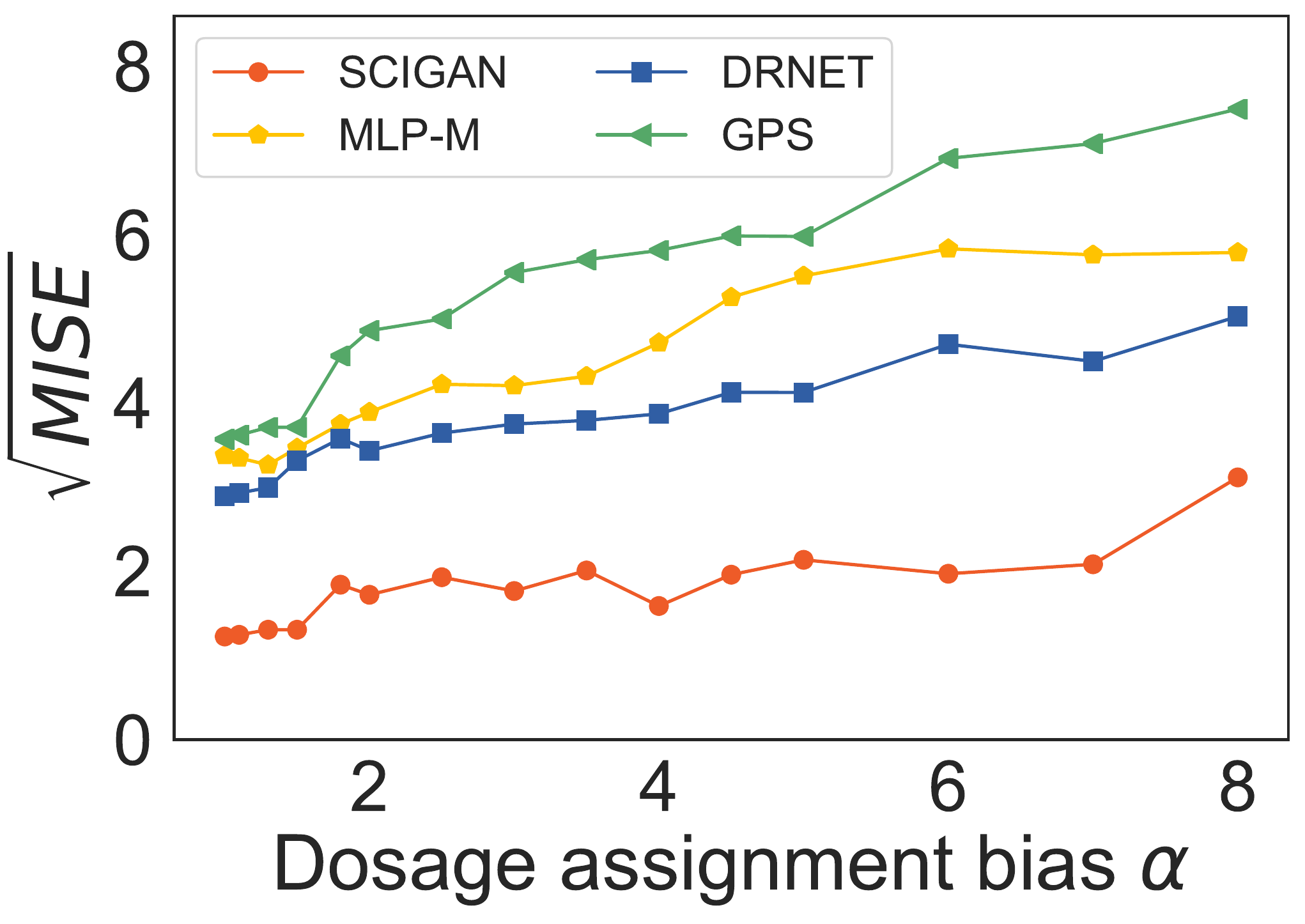} }
    \hfill
    \subfloat[Dosage bias]{\includegraphics[width=0.23\linewidth]{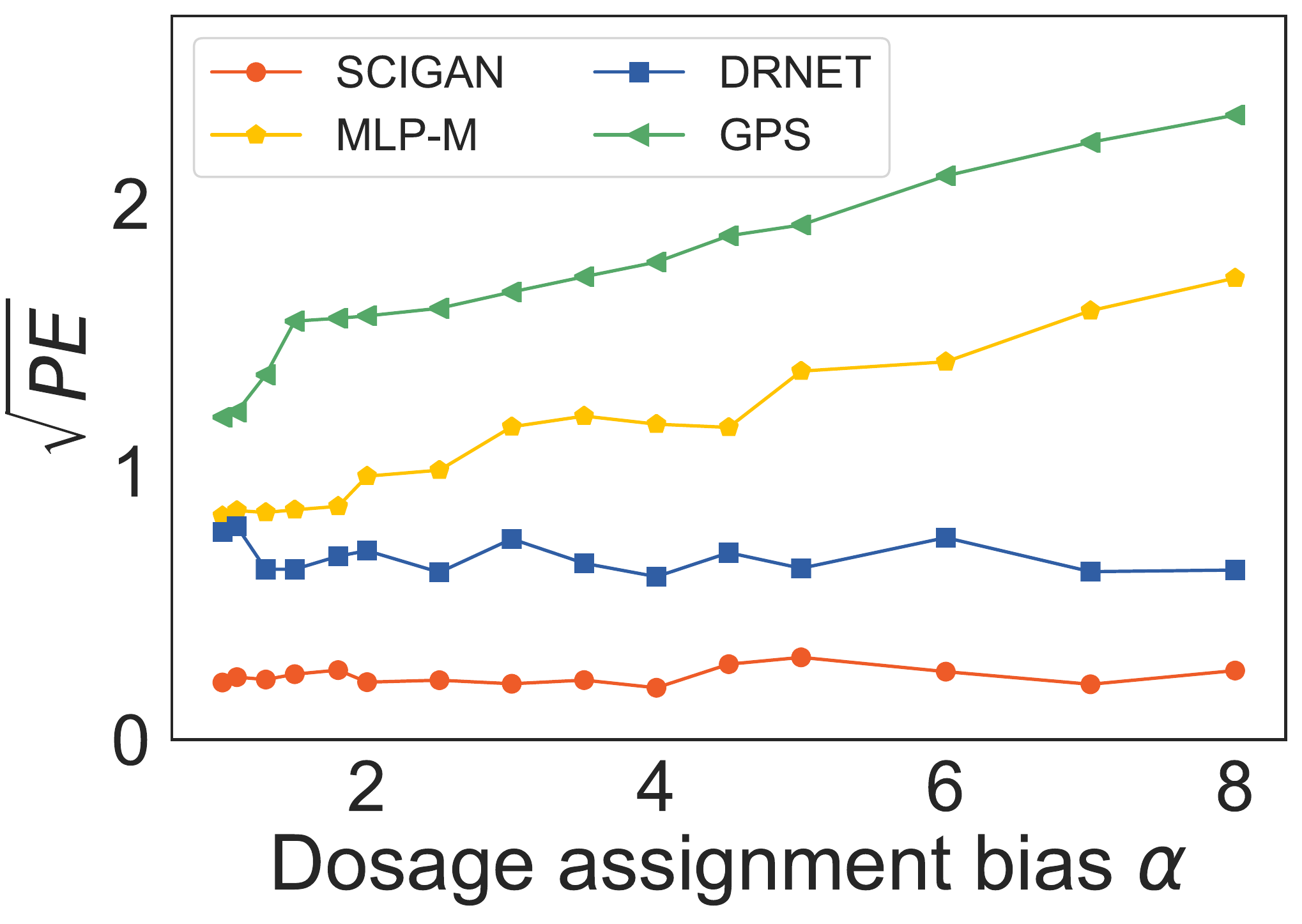}}
    \caption{Performance of the 4 methods on datasets with varying bias levels.}
    \vspace{-3mm}
    \label{fig:selection_bias_mise}
\end{figure}

\section{Conclusion}
\vspace{-0.1cm}
In this paper we proposed a novel framework for estimating response curves for continuous interventions from observational data. We provided theoretical justification for our use of a modified GAN framework, which introduced a novel hierarchical discriminator. While our approach is very flexible (it accepts multiple treatments potentially with/without dosage parameters), one limitation is that SCIGAN needs at least a few thousand training samples, as is generally the case with neural networks, and GANs in particular. As for future research, we used out-of-the-box methods for the invariant and equivariant layers but feel that more can be done to make these layers as expressive as possible for this model. Moreover, another important research direction is to extend this work to the temporal causal inference setting \cite{robins2000marginal, lim2018forecasting, bica2020estimating, bica2019time} and estimate counterfactual outcomes for sequences of continuous interventions. 


\section*{Broader Impact}
The impact of this problem in the healthcare setting is clear - being able to better estimate individualised responses to dosages will help us select treatments that result in improved patient outcomes. Moreover, clinicians and patients will often need to consider several different outcomes (such as potential side effects); better estimates of such outcomes allow the patients to make a more informed decision that is suitable for them.

Going beyond predictions by using causal inference methods to estimate the effect of interventions will result in more accurate and robust estimates and thus create more reliable components for use as part of decision support systems. Much of the recent work in causal inference has focused on binary or categorical treatments. Nevertheless, continuous interventions arise in many practical scenarios and building reliable methods for estimating their effects is paramount. We believe that our proposed model, SCIGAN, represents an important step forward in this direction. Nevertheless, we acknowledge the fact that the work presented in this paper is on the theoretical side and significant testing, potentially through clinical trials, will be needed before such methods can be used in practice, particularly due to the life-threatening implications of incorrect estimates. The risk of incorrectly assigning treatments can be significantly mitigated by ensuring that such models are used as {\em support} systems alongside clinicians, rather than instead of clinicians.

All work towards better estimating and understanding interventions can be used negatively, where someone wishing to cause harm can use the estimated outcomes to select the worst outcome.

\section*{Acknowledgments}
We would like to thank the reviewers for their valuable feedback. The research presented in this paper was supported by The Alan Turing Institute, under the EPSRC grant EP/N510129/1 and by the US Office of Naval Research (ONR), NSF 1722516.

\bibliography{refs.bib}

\begin{thebibliography}{10}

\bibitem{bertsimas2017personalized}
Dimitris Bertsimas, Nathan Kallus, Alexander~M Weinstein, and Ying~Daisy Zhuo.
\newblock Personalized diabetes management using electronic medical records.
\newblock {\em Diabetes Care}, 40(2):210--217, 2017.

\bibitem{alaa2017deep}
Ahmed~M Alaa, Michael Weisz, and Mihaela Van Der~Schaar.
\newblock Deep counterfactual networks with propensity-dropout.
\newblock {\em ICML 2017 - Workshop on Principled Approaches to Deep Learning},
  2017.

\bibitem{alaa2017bayesian}
Ahmed~M Alaa and Mihaela van~der Schaar.
\newblock Bayesian inference of individualized treatment effects using
  multi-task gaussian processes.
\newblock In {\em Advances in Neural Information Processing Systems}, pages
  3424--3432, 2017.

\bibitem{athey2016recursive}
Susan Athey and Guido Imbens.
\newblock Recursive partitioning for heterogeneous causal effects.
\newblock {\em Proceedings of the National Academy of Sciences},
  113(27):7353--7360, 2016.

\bibitem{wager2018estimation}
Stefan Wager and Susan Athey.
\newblock Estimation and inference of heterogeneous treatment effects using
  random forests.
\newblock {\em Journal of the American Statistical Association},
  113(523):1228--1242, 2018.

\bibitem{yoon2018ganite}
Jinsung Yoon, James Jordon, and Mihaela van~der Schaar.
\newblock {GANITE}: Estimation of individualized treatment effects using
  generative adversarial nets.
\newblock {\em International Conference on Learning Representations (ICLR)},
  2018.

\bibitem{alaa2018limits}
Ahmed Alaa and Mihaela Schaar.
\newblock Limits of estimating heterogeneous treatment effects: Guidelines for
  practical algorithm design.
\newblock In {\em International Conference on Machine Learning}, pages
  129--138, 2018.

\bibitem{zhang2020learning}
Yao Zhang, Alexis Bellot, and Mihaela van~der Schaar.
\newblock Learning overlapping representations for the estimation of
  individualized treatment effects.
\newblock {\em International Conference on Artificial Intelligence and
  Statistics (AISTATS)}, 2020.

\bibitem{bica2020real}
Ioana Bica, Ahmed~M Alaa, Craig Lambert, and Mihaela van~der Schaar.
\newblock From real-world patient data to individualized treatment effects
  using machine learning: Current and future methods to address underlying
  challenges.
\newblock {\em Clinical Pharmacology \& Therapeutics}, 2020.

\bibitem{dopp2013high}
Donna D{\"o}pp-Zemel and AB~Johan Groeneveld.
\newblock High-dose norepinephrine treatment: determinants of mortality and
  futility in critically ill patients.
\newblock {\em American Journal of Critical Care}, 22(1):22--32, 2013.

\bibitem{wang2017cardiac}
Kyle Wang, Michael~J Eblan, Allison~M Deal, Matthew Lipner, Timothy~M Zagar,
  Yue Wang, Panayiotis Mavroidis, Carrie~B Lee, Brian~C Jensen, Julian~G
  Rosenman, et~al.
\newblock Cardiac toxicity after radiotherapy for stage {III} non--small-cell
  lung cancer: Pooled analysis of dose-escalation trials delivering 70 to 90
  {Gy}.
\newblock {\em Journal of Clinical Oncology}, 35(13):1387, 2017.

\bibitem{zame2020machine}
William~R Zame, Ioana Bica, Cong Shen, Alicia Curth, Hyun-Suk Lee, Stuart
  Bailey, James Weatherall, David Wright, Frank Bretz, and Mihaela van~der
  Schaar.
\newblock Machine learning for clinical trials in the era of covid-19.
\newblock {\em Statistics in Biopharmaceutical Research}, pages 1--12, 2020.

\bibitem{ite_tutorial}
Peter Spirtes.
\newblock A tutorial on causal inference.
\newblock 2009.

\bibitem{goodfellow2014generative}
Ian Goodfellow, Jean Pouget-Abadie, Mehdi Mirza, Bing Xu, David Warde-Farley,
  Sherjil Ozair, Aaron Courville, and Yoshua Bengio.
\newblock Generative adversarial nets.
\newblock In {\em Advances in Neural Information Processing Systems}, pages
  2672--2680, 2014.

\bibitem{zaheer2017deepsets}
Manzil Zaheer, Satwik Kottur, Siamak Ravanbakhsh, Barnabas Poczos, Ruslan~R
  Salakhutdinov, and Alexander~J Smola.
\newblock Deep sets.
\newblock In {\em Advances in Neural Information Processing Systems}, pages
  3391--3401, 2017.

\bibitem{imbens2000role}
Guido~W Imbens.
\newblock The role of the propensity score in estimating dose-response
  functions.
\newblock {\em Biometrika}, 87(3):706--710, 2000.

\bibitem{imai2004causal}
Kosuke Imai and David~A Van~Dyk.
\newblock Causal inference with general treatment regimes: Generalizing the
  propensity score.
\newblock {\em Journal of the American Statistical Association},
  99(467):854--866, 2004.

\bibitem{hirano2004propensity}
Keisuke Hirano and Guido~W Imbens.
\newblock The propensity score with continuous treatments.
\newblock {\em Applied Bayesian Modeling and Causal Inference from
  Incomplete-Data Perspectives}, 226164:73--84, 2004.

\bibitem{schwab2019learning}
Patrick Schwab, Lorenz Linhardt, Stefan Bauer, Joachim~M Buhmann, and Walter
  Karlen.
\newblock Learning counterfactual representations for estimating individual
  dose-response curves.
\newblock {\em arXiv preprint arXiv:1902.00981}, 2019.

\bibitem{shalit2017estimating}
Uri Shalit, Fredrik~D Johansson, and David Sontag.
\newblock Estimating individual treatment effect: Generalization bounds and
  algorithms.
\newblock In {\em Proceedings of the 34th International Conference on Machine
  Learning-Volume 70}, pages 3076--3085. JMLR. org, 2017.

\bibitem{rubin1984bayesianly}
Donald~B Rubin.
\newblock Bayesianly justifiable and relevant frequency calculations for the
  applies statistician.
\newblock {\em The Annals of Statistics}, pages 1151--1172, 1984.

\bibitem{weinstein2013cancer}
John~N Weinstein, Eric~A Collisson, Gordon~B Mills, Kenna R~Mills Shaw, Brad~A
  Ozenberger, Kyle Ellrott, Ilya Shmulevich, Chris Sander, Joshua~M Stuart,
  Cancer Genome Atlas~Research Network, et~al.
\newblock The cancer genome atlas pan-cancer analysis project.
\newblock {\em Nature Genetics}, 45(10):1113, 2013.

\bibitem{johansson2016learning}
Fredrik Johansson, Uri Shalit, and David Sontag.
\newblock Learning representations for counterfactual inference.
\newblock In {\em International Conference on Machine Learning}, pages
  3020--3029, 2016.

\bibitem{mimiciii}
Alistair~EW Johnson, Tom~J Pollard, Lu~Shen, H~Lehman Li-wei, Mengling Feng,
  Mohammad Ghassemi, Benjamin Moody, Peter Szolovits, Leo~Anthony Celi, and
  Roger~G Mark.
\newblock Mimic-iii, a freely accessible critical care database.
\newblock {\em Scientific Data}, 3:160035, 2016.

\bibitem{silva2016observational}
Ricardo Silva.
\newblock Observational-interventional priors for dose-response learning.
\newblock In {\em Advances in Neural Information Processing Systems}, pages
  1561--1569, 2016.

\bibitem{robins2000marginal}
James~M Robins, Miguel~Angel Hernan, and Babette Brumback.
\newblock Marginal structural models and causal inference in epidemiology,
  2000.

\bibitem{lim2018forecasting}
Bryan Lim, Ahmed Alaa, and Mihaela van~der Schaar.
\newblock Forecasting treatment responses over time using recurrent marginal
  structural networks.
\newblock In {\em Advances in Neural Information Processing Systems}, pages
  7493--7503, 2018.

\bibitem{bica2020estimating}
Ioana Bica, Ahmed~M Alaa, James Jordon, and Mihaela van~der Schaar.
\newblock Estimating counterfactual treatment outcomes over time through
  adversarially balanced representations.
\newblock In {\em International Conference on Learning Representations}, 2020.

\bibitem{bica2019time}
Ioana Bica, Ahmed~M Alaa, and Mihaela van~der Schaar.
\newblock Time series deconfounder: Estimating treatment effects over time in
  the presence of hidden confounders.
\newblock {\em International Conference on Machine Learning (ICML)}, 2020.

\bibitem{stoehlmacher2004multivariate}
J~Stoehlmacher, DJ~Park, W~Zhang, D~Yang, S~Groshen, S~Zahedy, and HJ~Lenz.
\newblock A multivariate analysis of genomic polymorphisms: Prediction of
  clinical outcome to 5-{FU/Oxaliplatin} combination chemotherapy in refractory
  colorectal cancer.
\newblock {\em British Journal of Cancer}, 91(2):344, 2004.

\bibitem{qian2011performance}
Min Qian and Susan~A Murphy.
\newblock Performance guarantees for individualized treatment rules.
\newblock {\em Annals of Statistics}, 39(2):1180, 2011.

\bibitem{crump2008nonparametric}
Richard~K Crump, V~Joseph Hotz, Guido~W Imbens, and Oscar~A Mitnik.
\newblock Nonparametric tests for treatment effect heterogeneity.
\newblock {\em The Review of Economics and Statistics}, 90(3):389--405, 2008.

\bibitem{shi2019adapting}
Claudia Shi, David~M Blei, and Victor Veitch.
\newblock Adapting neural networks for the estimation of treatment effects.
\newblock {\em arXiv preprint arXiv:1906.02120}, 2019.

\bibitem{chipman2010bart}
Hugh~A Chipman, Edward~I George, Robert~E McCulloch, et~al.
\newblock {BART}: Bayesian additive regression trees.
\newblock {\em The Annals of Applied Statistics}, 4(1):266--298, 2010.

\bibitem{kallus2017recursive}
Nathan Kallus.
\newblock Recursive partitioning for personalization using observational data.
\newblock In {\em Proceedings of the 34th International Conference on Machine
  Learning-Volume 70}, pages 1789--1798. JMLR. org, 2017.

\bibitem{li2017matching}
Sheng Li and Yun Fu.
\newblock Matching on balanced nonlinear representations for treatment effects
  estimation.
\newblock In {\em Advances in Neural Information Processing Systems}, pages
  929--939, 2017.

\bibitem{yao2018representation}
Liuyi Yao, Sheng Li, Yaliang Li, Mengdi Huai, Jing Gao, and Aidong Zhang.
\newblock Representation learning for treatment effect estimation from
  observational data.
\newblock In {\em Advances in Neural Information Processing Systems}, pages
  2633--2643, 2018.

\bibitem{swaminathan2015batch}
Adith Swaminathan and Thorsten Joachims.
\newblock Batch learning from logged bandit feedback through counterfactual
  risk minimization.
\newblock {\em The Journal of Machine Learning Research}, 16(1):1731--1755,
  2015.

\bibitem{kallus2018policy}
Nathan Kallus and Angela Zhou.
\newblock Policy evaluation and optimization with continuous treatments.
\newblock {\em arXiv preprint arXiv:1802.06037}, 2018.

\bibitem{bertsimas2018optimization}
Dimitris Bertsimas and Christopher McCord.
\newblock Optimization over continuous and multi-dimensional decisions with
  observational data.
\newblock In {\em Advances in Neural Information Processing Systems}, pages
  2962--2970, 2018.

\bibitem{chernozhukov2019semi}
Victor Chernozhukov, Mert Demirer, Greg Lewis, and Vasilis Syrgkanis.
\newblock Semi-parametric efficient policy learning with continuous actions.
\newblock In {\em Advances in Neural Information Processing Systems}, pages
  15039--15049, 2019.

\bibitem{galagate2016causal}
Douglas Galagate.
\newblock {\em Causal Inference with a Continuous Treatment and Outcome:
  Alternative Estimators for Parametric Dose-Response function with
  Applications.}
\newblock PhD thesis, 2016.

\bibitem{bergstra2012random}
James Bergstra and Yoshua Bengio.
\newblock Random search for hyper-parameter optimization.
\newblock {\em Journal of Machine Learning Research}, 13(Feb):281--305, 2012.

\end{thebibliography}
\bibliographystyle{unsrt}

\newpage
\appendix

\section{Expanded Related Works} \label{app:exprelwork}
Most methods for performing causal inference in the static setting focus on the scenario with two or multiple treatment options and no dosage parameter. The approaches taken by such methods to estimate the treatment effects involve either building a separate regression model for each treatment \cite{stoehlmacher2004multivariate, qian2011performance, bertsimas2017personalized} or using the treatment as a feature and adjusting for the imbalance between the different treatment populations. The former does not generalise to the dosage setting due to the now infinite number of possible treatments available. In the latter case, methods for handling the selection bias involve propensity weighting \cite{crump2008nonparametric, alaa2017deep, shi2019adapting}, building sub-populations using tree based methods \cite{chipman2010bart, athey2016recursive, wager2018estimation, kallus2017recursive} or building balancing representations between patients receiving the different treatments \cite{johansson2016learning, shalit2017estimating, li2017matching,yao2018representation}. An additional approach involves modelling the data distribution of the factual and counterfactual outcomes \cite{alaa2017bayesian,  yoon2018ganite}.  

\cite{silva2016observational} leverages observational and interventional data to estimate the effects of discrete dosages for a single treatment. In particular, \cite{silva2016observational} uses observational data to construct a non-stationary covariance function and develop a hierarchical Gaussian process prior to build a distribution over the dose response curve. Then, controlled interventions are employed to learn a non-parametric affine transform to reshape this distribution. The setting in \cite{silva2016observational} differs significantly from ours as we do not assume access to any interventional data.

\subsection{Comparison with GANITE}
Naive attempts to extend \cite{yoon2018ganite} to the continuous setting might involve: (1) discretising the continuous space of interventions; (2) somehow passing entire response curves to the discriminator and asking it to identify the point on the curve that corresponds to the factual outcome. Naturally, discretisation comes with a cost. If the discretisation is too coarse, the response curves will not be well-approximated. On the other hand, we show experimentally that GANITE is incapable of handling a high number of discrete interventions (corresponding to having a finer discretisation). In fact, although SCIGAN was designed for continuous interventions, it can be applied in the discrete setting and we show that it outperforms GANITE when the (discrete) parameter space is not small.

For (2), the problem is in defining a mechanism for generating these response curves in a form that can be passed to the discriminator and ensuring the {\em continuity} of these curves around the factual outcome so that the discontinuity itself does not make identification trivial for the discriminator. To overcome this we define a discriminator that acts on a finite set of points from each generated response curve (rather than on entire curves), as shown in Fig. \ref{fig:curves}. From {\em among the chosen points}, the discriminator attempts to identify the factual one. The set of points is sampled randomly {\em each} time an input would be passed to the discriminator. As our discriminator will be acting on a {\em set} of random intervention-outcome pairs, we explicitly condition it to behave as a function on a set. In particular, we draw on ideas from \cite{zaheer2017deepsets} to ensure that its output does not depend on the {\em order} of its input.

In addition, for the setting in which there are multiple possible interventions that {\em each} have an associated continuous parameter (which is the main setting of the paper), we propose a {\em hierarchical} discriminator which breaks down the job of the discriminator into determining the factual intervention and determining the factual parameter using separate networks. We show in the experiments section that this approach significantly improves performance and is more stable than using a single network discriminator. In this setting, we also model the generator as a multi-task deep network capable of taking a continuous parameter as an input; this gives us the flexibility to learn heterogeneous response curves for the different interventions.

\subsection{Off-policy evaluation and policy optimization with continuous treatments}

A problem related to ours involves evaluating policies and learning optimal policies from logged data \cite{swaminathan2015batch}. In this context, several methods have been proposed for performing off-policy evaluation and policy optimization with continuous treatments \cite{kallus2018policy, bertsimas2018optimization, chernozhukov2019semi}. It should be emphasized that there is a significant difference between this line of research, which aims to find optimal policies, and the causal inference setting considered in this paper, where the aim is to learn the counterfactual patient outcomes under
all possible treatment options.

For off-policy policy evaluation with continuous treatments, \cite{kallus2018policy} propose a method based on inverse-propensity weighting to estimate the value function, i.e. cumulative reward of a target policy from observational data. However, their proposed method does not perform any type of ITE estimation and does not have any intermediate per-sample value estimates. Alternatively, \cite{chernozhukov2019semi} does perform intermediate estimation of the value function for each sample. However, their contribution is to construct a doubly-robust estimator for the value function which crucially depends on knowing (or having some prior knowledge of) the parametric form of the response curve’s dependency on the treatment. Without an assumed form, it is not possible to construct their estimator. Our proposed model, SCIGAN, does not assume any prior knowledge about the response curve.

The problem of individualised treatment effect estimation (which is our primary goal) is harder than off-policy evaluation as it involves estimating the outcomes (which may differ from the value function) for every sample and every possible action that could have been taken for that sample. After having learned these counterfactuals it is possible to perform policy evaluation, but note that the use cases for treatment effect estimation go beyond policy evaluation. In particular, for a given setting it may be far more beneficial to present a patient with their estimated outcome along with potential side effects and allow them (or the clinician) to come to a decision based on these several factors. In such a setting, each patient may have their own “internal” value function that depends on potential outcomes and side effects differently.

\section{Notation} \label{app:notation}
In the table below, we summarise the notation used in our paper. Note that realisations of random variables are denoted using lowercase and subscripts/superscripts used with vector-valued functions denotes their output at the position of the given subscript/superscript.

\begin{table*}[h]
    	\centering
    	\begin{adjustbox}{max width=\textwidth}
    	\begin{tabular}{cl}
    		\toprule
    		$\mathcal{X}$ & Feature space \\
    		$\mathcal{Y}$ & Outcome space \\
    		$\mathcal{T}$ & Intervention space \\
    		$\mathcal{W} = \{w_1, ..., w_k\}$ & Set of treatments \\
    		$\mathcal{D}_w$ & Dosage space for treatment $w \in \mathcal{W}$ \\
    		$\mathbf{X} \in \mathcal{X}$ & Features (random variable) \\
    		$Y : \mathcal{T} \to \mathcal{Y}$ & Potential outcome function (function-valued random variable) \\
    		$T_f = (W_f, D_f) \in \mathcal{T}$ & Factual/observed intervention (treatment-dosage pair) (random variable) \\
    		$Y_f \in \mathcal{Y}$ & Outcome corresponding to the observed intervention ($Y_f = Y(W_f, D_f)$) \\
    		$\mathbf{G}$ & Generator \\
    		$\mathbf{Z}$ & Random noise (for input to generator) (random variable) \\
    		$\hat{Y}_{cf} : \mathcal{T} \to \mathcal{Y}$ & Counterfactual outcome function induced by $\mathbf{G}$ \\
    		$\mathbf{D}$ & Discriminator \\
    		$\tilde{\mathcal{D}}_w = \{D^w_1, ..., D^w_{n_w}\}$ & Random (finite) subset of $\mathcal{D}_w$ \\
    		$n_w$ & Size of $\tilde{\mathcal{D}}_w$ (i.e. number of dosage levels passed to discriminator for treatment $w \in \mathcal{W}$) \\
    		$\tilde{\mathbf{Y}}_w = (D^w_i, \tilde{Y}^w_i)_{i = 1}^{n_w}$ & Vector of dosage-outcome pairs generated by $\mathbf{G}$ (and $Y_f$) using $\tilde{\mathcal{D}}_w$ \\
    		$\mathbf{D}_\mathcal{W}$ & Treatment discriminator \\
    		$\mathbf{D}_w$ & Dosage discriminator for treatment $w \in \mathcal{W}$ \\
    		$\mathbf{D}_H$ & Hierarchical discriminator defined by combining $\mathbf{D}_\mathcal{W}$ and $\mathbf{D}_w$ \\
    		$\mathbf{I}$ & Inference network \\
    		$\mathcal{L}$ & GAN loss \\
    		$\mathcal{L}_S$ & Supervised loss \\
    		
    		\bottomrule
    	\end{tabular}
    	\end{adjustbox}
    	\label{tab:notation}
    \end{table*}

\clearpage

\section{Proofs of Theoretical Results} \label{app:proofs}
In this section we prove Theorem \ref{thm:main}. Note that these results also apply to GANITE (with very minor modifications - the proofs are even simpler in the case of GANITE).

In order to prove Theorem \ref{thm:main}, we analyse the simpler minimax game defined by
\begin{equation} \label{eq:minimax}
    \min_\mathbf{G} \max_\mathbf{D} \mathcal{L}(\mathbf{D}, \mathbf{G}) + \lambda \mathcal{L}_S(\mathbf{G}) \,,
\end{equation}
which corresponds to the single discriminator model (instead of the hierarchical model) and then use this to prove our full result.

\setcounter{lemma}{0}
\setcounter{theorem}{0}

\begin{lemma} \label{lem:opt_d}
Fix $\mathbf{G}$ and $\tilde{\mathcal{D}} = \bigcup_w \tilde{\mathcal{D}}_w$. Let $p_{w, d}(\mathbf{y} | \mathbf{x}) = p_r(y_{w, d} | \mathbf{x}) p_\mathbf{G}(\mathbf{y}_{\neg w,d} | \mathbf{x}, y_{w, d})$ denote the induced joint density of outcomes when restricted to dosages in $\tilde{\mathcal{D}}$, where $p_r$ denotes the true density that generated the observed outcome and $p_\mathbf{G}$ denotes the density induced by $\mathbf{G}$ over the remaining dosages in $\tilde{\mathcal{D}}$. Then the optimal discriminator is
\begin{equation}
\mathbf{D}^*_{w, j}(\mathbf{x}, \mathbf{y}) = \frac{\tilde{p}(w, d_j | \mathbf{x}) p_{w, d_j}(\mathbf{y}|\mathbf{x})}{\sum_{w' \in \mathcal{W}} \sum_{i = 1}^{n_w} \tilde{p}(w', d_i | \mathbf{x})p_{w', d_i}(\mathbf{y} | \mathbf{x})}
\end{equation}
where $\tilde{p}$ is the $\tilde{\mathcal{D}}$-restricted propensity given by $\tilde{p}(w, d_j | \mathbf{x}) = p(w | \mathbf{x}) (p(d_j | \mathbf{x}, w) / \sum_{i = 1}^{n_w} p(d_i | \mathbf{x}, w))$.
\end{lemma}

\begin{proof}
Fix $\mathbf{G}$ and $\tilde{\mathcal{D}} = \bigcup_w \tilde{\mathcal{D}}_w$. The optimal discriminator is given by $\arg \min_{\mathbf{D}} \mathcal{L}(\mathbf{D}, \mathbf{G})$. We have
\begin{align}
    \mathcal{L}(\mathbf{D}, \mathbf{G}) &= \mathbb{E} \Bigg[ \sum_{w \in \mathcal{W}} \sum_{d \in \tilde{\mathcal{D}}_w} \mathbb{I}_{\{T_f = (w, d)\}} \log \mathbf{D}^{w, d}(\mathbf{X}, \tilde{\mathbf{Y}}) + \mathbb{I}_{\{T_f \neq (w, d)\}} \log (1 - \mathbf{D}^{w, d}(\mathbf{X}, \tilde{\mathbf{Y}})) \Bigg] \\
    &\begin{aligned}
    =\mathbb{E}_{\tilde{\mathcal{D}}} \Bigg[ \sum_{w \in \mathcal{W}} &\sum_{d \in \tilde{\mathcal{D}}_w} \int_{(\mathbf{x}, \mathbf{y})} \tilde{p}(w, d | \mathbf{x})p_{w, d}(\mathbf{y} | \mathbf{x}) \log \mathbf{D}^{w, d}(\mathbf{x}, \mathbf{y}) \\
    &+ \bigg(\sum_{w', d' \neq w, d} \tilde{p}(w', d' | \mathbf{x})p_{w', d'}(\mathbf{y} | \mathbf{x})\bigg) \log (1 - \mathbf{D}^{w, d}(\mathbf{x}, \mathbf{y})) p(\mathbf{x}) d\mathbf{y} d\mathbf{x} \Bigg]
    \end{aligned}
\end{align}
where we have taken the (conditional on $\tilde{\mathcal{D}}$) expectations inside the sums and replaced indicator functions with densities as appropriate. We now note that $a \log p + b \log(1 - p)$ for $p \in (0, 1)$ has a unique maximum at $p = \frac{a}{a+b}$, thus implying that the integrand is maximised when
\begin{equation} \label{eq:proof_opt_d}
    \mathbf{D}^{w, d}(\mathbf{x}, \mathbf{y}) = \frac{\tilde{p}(w, d | \mathbf{x})p_{w, d}(\mathbf{y}| \mathbf{x})}{\tilde{p}(w, d | \mathbf{x})p_{w, d}(\mathbf{y}| \mathbf{x}) + \sum_{w', d' \neq w, d} \tilde{p}(w', d' | \mathbf{x})p_{w', d'}(\mathbf{y} | \mathbf{x})} \,.
\end{equation}
This gives the required result.
\end{proof}

Using Lemma \ref{lem:opt_d} we can now show that the optimal solution to our single discriminator game is when the marginal distributions of the generated counterfactuals are equal to the true counterfactuals. Importantly, this suffices for estimating $\mu$ since the expectation is only concerned with the marginal distribution of $Y(w, d)$.

\begin{lemma} \label{lem:single_d}
The global minimum of the minimax game defined by $\min_\mathbf{G} \max_\mathbf{D} \mathcal{L}(\mathbf{D}, \mathbf{G})$ is achieved if and only if for all $\tilde{\mathcal{D}}_w$, for all $w, w' \in \mathcal{W}$ and for all $d \in \tilde{\mathcal{D}}$, $d' \in \tilde{\mathcal{D}}_{w'}$ we have that
\begin{equation}
    p_{w, d}(\mathbf{y} | \mathbf{x}) = p_{w', d'}(\mathbf{y} | \mathbf{x})
\end{equation}
which in turn implies that for any treatment-dosage pair $(w, d) \in \mathcal{T}$ we have that the generated counterfactual for outcome $(w, d)$ for any sample (that was not assigned $(w, d)$) has the same (marginal) distribution (conditional on the features) as the true marginal distribution for that outcome.
\end{lemma}

\begin{proof}
For fixed $\tilde{\mathcal{D}}$ and $\mathbf{x}$ we note that by substituting the optimal discriminator into $\mathcal{L}(\mathbf{D}, \mathbf{G})$ and subtracting $\sum_{w \in \mathcal{W}} \sum_{i = 1}^{n_w} \log \tilde{p}(w, d_i | \mathbf{x})$ (which is independent of $\mathbf{G}$) we obtain
\begin{align}
    \mathcal{L}(&\mathbf{D}^*, \mathbf{G}) - \int_\mathbf{x} \Bigg(\sum_{w \in \mathcal{W}} \sum_{i = 1}^{n_w} \log \tilde{p}(w, d_i | \mathbf{x}) \Bigg) p(\mathbf{x}) d\mathbf{x} \\
    &= \mathbb{E}_{\tilde{\mathcal{D}}} \int_\mathbf{x} \text{KL}\Big(p_{w, d}(\mathbf{y} | \mathbf{x}) || \hat{p}(\mathbf{y}| \mathbf{x})\Big) + \text{KL}\Big(\frac{1}{1 - \tilde{p}(w, d | \mathbf{x})} \sum_{t' \neq (w, d)} \tilde{p}(t' | \mathbf{x}) p_{t'}(\mathbf{y} | \mathbf{x}) || \hat{p}(\mathbf{y} | \mathbf{x})\Big) d\mathbf{x}
\end{align}
where KL is the KL divergence and $\hat{p}(\mathbf{y} | \mathbf{x}) = \sum_{t \in \tilde{\mathcal{T}}} \tilde{p}(t | \mathbf{x}) p_t(\mathbf{y} | \mathbf{x})$ where $\tilde{\mathcal{T}}$ is the restriction of $\mathcal{T}$ to the dosages in $\tilde{\mathcal{D}}$. We then note that the KL divergence is minimised if and only if the two densities are equal, and we note by definition of $\hat{p}$ this occurs if and only if $p_{w, d}(\mathbf{y}| \mathbf{x}) = p_{w', d'}(\mathbf{y} | \mathbf{x})$ for all $w, d, w', d'$. This also directly implies that the marginal distributions for any fixed treatment-dosage pair agree for all factually observed treatments. In particular, if a sample received treatment $t' \neq t$, we have that the counterfactual generated for $t$ for this sample has the same distribution as the true data generating distribution.
\end{proof}

Finally, we prove the following result, from which Theorem \ref{thm:main} follows immediately.

\begin{theorem}
An optimal solution to the game defined by Equations \ref{eq:hierD_H} - \ref{eq:hierD_w} is also an optimal solution to the game defined by Equation \ref{eq:minimax} if the response curves generated by the generator for different treatments are conditionally independent given the features.
\end{theorem}

\begin{proof}
To prove this result, it suffices to show that for fixed $\mathbf{G}$, $\mathbf{D}_H^* = \mathbf{D}^*$. To show this, we observe that by the same arguments as given for Lemma \ref{lem:opt_d}, we have the following:
\begin{align} \label{eq:opt_tr_d}
    {\mathbf{D}^w_\mathcal{W}}^*(\mathbf{x}, \mathbf{y}) &= \frac{p(w | \mathbf{x}) \Big(\sum_{i = 1}^{n_w} \tilde{p}(d_i | \mathbf{x}, w) p_{w, d_i}(\mathbf{y} | \mathbf{x})\Big)}{\sum_{w' \in \mathcal{W}} \Big(p(w' | \mathbf{x}) \sum_{i = 1}^{n_w} \tilde{p}(d_i | \mathbf{x}, w)\Big)} \\ \label{eq:opt_dos_d}
    {\mathbf{D}^j_w}^*(\mathbf{x}, \mathbf{y}_w) &= \frac{\tilde{p}(d_j | \mathbf{x}, w)p_{w, d_j}(\mathbf{y}_w | \mathbf{x})}{\sum_{i=1}^{n_w} \tilde{p}(d_i | \mathbf{x}, w)p_{w, d_i}(\mathbf{y}_w | \mathbf{x})}
\end{align}
where $\mathbf{y}_w$ is the restriction of $\mathbf{y}$ to the outcomes corresponding to treatment $w$. By multiplying (\ref{eq:opt_dos_d}) by $\frac{p_{w, d_j}(\mathbf{y}_{\neq w} | \mathbf{y}, \mathbf{x})}{p_{w, d_j}(\mathbf{y}_{\neq w} | \mathbf{y}, \mathbf{x})}$ we obtain
\begin{equation} \label{eq:opt_dos_d_adj}
    {\mathbf{D}^j_w}^*(\mathbf{x}, \mathbf{y}_w) = \frac{\tilde{p}(d_j | \mathbf{x}, w)p_{w, d_j}(\mathbf{y} | \mathbf{x})}{\sum_{i=1}^{n_w} \tilde{p}(d_i | \mathbf{x}, w)p_{w, d_i}(\mathbf{y} | \mathbf{x})}
\end{equation}
since the conditional independence assumption implies that $p_{w, d_j}(\mathbf{y}_{\neq w} | \mathbf{y}, \mathbf{x}) = p_{w, d_i}(\mathbf{y}_{\neq w} | \mathbf{y}, \mathbf{x})$ for all $i, j = 1, ..., n_w$.
Multiplying (\ref{eq:opt_tr_d}) and (\ref{eq:opt_dos_d_adj}) together to get $\mathbf{D}_H^{w, j}$, we notice that the denominator in (\ref{eq:opt_dos_d_adj}) cancels with the bracketed term of the numerator in (\ref{eq:opt_tr_d}) to give
\begin{align}
    {\mathbf{D}^{*}_H}^{w, j} &= \frac{p(w | \mathbf{x}) \tilde{p}(d_j | \mathbf{x}, w) p_{w, d_j}(\mathbf{y} | \mathbf{x})}{\sum_{w' \in \mathcal{W}} \Big(p(w' | \mathbf{x}) \sum_{i = 1}^{n_w} \tilde{p}(d_i | \mathbf{x}, w)\Big)} \\
    &= \frac{\tilde{p}(w, d_j | \mathbf{x}) p_{w, d_j}(\mathbf{y}|\mathbf{x})}{\sum_{w' \in \mathcal{W}} \sum_{i = 1}^{n_w} \tilde{p}(w', d_i | \mathbf{x})p_{w', d_i}(\mathbf{y} | \mathbf{x})}
\end{align}
which is equal to the optimal discriminator for the single loss given in Lemma \ref{lem:opt_d}.
\end{proof}

\clearpage

\section{Counterfactual Generator Pseudo-code} \label{app:pseudo}
\begin{algorithm}[H]
	\begin{algorithmic}[1] 
		\State \textbf{Input:} dataset $\mathcal{C} = \{(\mathbf{x}^i, t_f^i, y_f^i) : i = 1, ..., N\}$, batch size $n_{mb}$, number of dosages per treatment $n_d$, number of discriminator updates per iteration $n_D$, number of generator updates per iteration $n_G$,  dimensionality of noise $n_z$, learning rate $\alpha$
		\State \textbf{Initialize:} $\theta_G$, $\theta_\mathcal{W}$, $\{\theta_w\}_{w \in \mathcal{W}}$
		\While{$\mathbf{G}$ has not converged}
		\Statex Discriminator updates
		\For{$i = 1, ..., n_D$}
		\State Sample $(\mathbf{x}_1, (w_1, d_1), y_1), ..., (\mathbf{x}_{n_{mb}}, (w_{n_{mb}}, d_{n_{mb}}), y_{n_{mb}})$ from $\mathcal{C}$
		\State Sample generator noise $\mathbf{z}_j = (z_1^j, ..., z_{n_z}^j)$ from $\text{Unif}([0, 1]^{n_z})$ for $j = 1, ..., n_{mb}$
		\For{$w \in \mathcal{W}$}
		\For{$j = 1, ..., n_{mb}$}
		\State Sample $\tilde{D}_w^j = (d^{w, j}_1, ..., d^{w, j}_{n_d})$ independently and uniformly from $(\mathcal{D}_w)^{n_d}$
		\State Set $\tilde{\mathbf{y}}_w^j$ according to Eq. \ref{eq:tildey} 
		\EndFor
		\State Calculate gradient of dosage discriminator loss
		\begin{small}
		\begin{equation*}
		    g_w \gets \nabla_{\theta_w} - \left[\sum_{\{j : w_j = w\}} \sum_{k = 1}^{n_d} \mathbb{I}_{\{d_j = d^{w, j}_k\}} \log \mathbf{D}_w(\mathbf{x}_j, \tilde{\mathbf{y}}_w^j) + \mathbb{I}_{\{d_j \neq d^{w, j}_k\}} \log(1 - \mathbf{D}_w(\mathbf{x}_j, \tilde{\mathbf{y}}_w^j)) \right]
		\end{equation*}
		\end{small}
		\State Update dosage discriminator parameters $\theta_w \gets \theta_w + \alpha g_w$
		\EndFor
		\State Set $\tilde{\mathbf{y}}_j = (\tilde{\mathbf{y}}_w^j)_{w \in \mathcal{W}}$
		\State Calculate gradient of treatment discriminator loss
		\begin{small}
		\begin{equation*}
		    g_\mathcal{W} \gets \nabla_{\theta_\mathcal{W}} - \left[\sum_{j = 1}^{n_{mb}} \sum_{w \in \mathcal{W}} \mathbb{I}_{\{w_j = w\}} \log \mathbf{D}_\mathcal{W}(\mathbf{x}_j, \tilde{\mathbf{y}}_j) + \mathbb{I}_{\{w_j \neq w\}} \log(1 - \mathbf{D}_\mathcal{W}(\mathbf{x}_j, \tilde{\mathbf{y}}_j)) \right]
		\end{equation*}
		\end{small}
		\State Update treatment discriminator parameters $\theta_\mathcal{W} \gets \theta_\mathcal{W} + \alpha g_\mathcal{W}$
		\EndFor
		\State Generator updates
		\For{$i = 1, ..., n_G$}
		\State Sample $(\mathbf{x}_1, (w_1, d_1), y_1), ..., (\mathbf{x}_{n_{mb}}, (w_{n_{mb}}, d_{n_{mb}}), y_{n_{mb}})$ from $\mathcal{C}$
		\State Sample generator noise $\mathbf{z}_j = (z_1^j, ..., z_{n_z}^j)$ from $\text{Unif}([0, 1]^{n_z})$ for $j = 1, ..., n_{mb}$
		\State Sample $(\tilde{D}_w^j)_{w \in \mathcal{W}}$ from $\Pi_{w \in \mathcal{W}} (\mathcal{D}_w)^{n_d}$ for $j = 1, ..., n_{mb}$
		\State Set $\tilde{\mathbf{y}}$ according to Eq. \ref{eq:tildey}
		\State Calculate gradient of generator loss
		\begin{small}
		\begin{eqnarray*}
		    g_G \gets \nabla_{\theta_G} \Bigg[\sum_{j = 1}^{n_{mb}} \sum_{w \in \mathcal{W}} \sum_{l = 1}^{n_d} \mathbb{I}_{\{w_j = w, d_j = d^{w, j}_l\}} \log (\mathbf{D}_\mathcal{W}^w(\mathbf{x}_j, \tilde{\mathbf{y}}_j)_w \times \mathbf{D}^l_w(\mathbf{x}_j, \tilde{\mathbf{y}}_w^j)_l) \\
		    + \mathbb{I}_{\{w_j \neq w, d_j \neq d^{w, j}_l\}} \log(1 - (\mathbf{D}^w_\mathcal{W}(\mathbf{x}_j, \tilde{\mathbf{y}}_j) \times D^l_w(\mathbf{x}_j, \tilde{\mathbf{y}}_w^j))) \Bigg]
		\end{eqnarray*}
		\end{small}
		\State Update generator parameters $\theta_G \gets \theta_G + \alpha g_G$
		\EndFor
		\EndWhile
		\State \textbf{Output:} $\mathbf{G}$
	\end{algorithmic}
	\caption{Training of the generator in SCIGAN}\label{alg:generator}
\end{algorithm}

\clearpage

\section{Inference Network} \label{app:infnet}
To generate dose-response curves for new samples, we learn an inference network, $\mathbf{I} : \mathbf{X} \times \mathcal{T} \to \mathcal{Y}$. This inference network is trained using the original dataset and the learned counterfactual generator. As with the training of the generator and discriminator, we train using a random set of dosages, $\tilde{\mathcal{D}}_w$. The loss is given by
\begin{equation}
    \mathcal{L}_I (\mathbf{I}) = \mathbb{E}\Bigg[\sum_{w \in \mathcal{W}} \sum_{d \in \tilde{\mathcal{D}}_w} (\tilde{Y}(w, d) - \mathbf{I}(\mathbf{X}, (w,d)))^2\Bigg]\,,
\end{equation}
where $\tilde{Y}(w, d)$ is $Y_f$ if $T_f = (w, d)$ or given by the generator if $T_f \neq (w, d)$. The expectation is taken over $\mathbf{X}, T_f, Y_f, \mathbf{Z}$ and $\tilde{\mathcal{D}}_w$.

\subsection{Pseudo-code for training the Inference Network}
\begin{algorithm}
	\begin{algorithmic}[1] 
		\State \textbf{Input:} dataset $\mathcal{C} = \{(\mathbf{x}^i, t_f^i, y_f^i) : i = 1, ..., N\}$, trained generator $\mathbf{G}$, batch size $n_{mb}$, number of dosages per treatment $n_d$,  dimensionality of noise $n_z$, learning rate $\alpha$
		\State \textbf{Initialize:} $\theta_I$
		\While {$\mathbf{I}$ has not converged}
		    \State Sample $(\mathbf{x}_1, (w_1, d_1), y_1), ..., (\mathbf{x}_{n_{mb}}, (w_{n_{mb}}, d_{n_{mb}}), y_{n_{mb}})$ from $\mathcal{C}$
		    \State Sample generator noise $\mathbf{z}_j = (z_1^j, ..., z_{n_z}^j)$ from $\text{Unif}([0, 1]^{n_z})$ for $j = 1, ..., n_{mb}$
		    
		    \For{$j = 1, ..., n_{mb}$}
		        \For{$w \in \mathcal{W}$}
    		        \State Sample $\tilde{D}_w^j = (d^{w, j}_1, ..., d^{w, j}_{n_d})$ independently and uniformly from $(\mathcal{D}_w)^{n_d}$
		            \State Set $\tilde{\mathbf{y}}_w^j$ according to Eq. \ref{eq:tildey}  \ref{eq:tildey}
    		    \EndFor
            \EndFor
            \State Calculate gradient of inference network loss
        	\begin{eqnarray*}
        	    g_I \gets - \nabla_{\theta_I} \Bigg[\sum_{j = 1}^{n_{mb}} \sum_{w \in \mathcal{W}} \sum_{l = 1}^{n_d} (\tilde{\mathbf{y}}^j_w)_l - \mathbf{I}(\mathbf{x}_j, (w, d_l^{w, j}))^2 \Bigg] \\ 
        	\end{eqnarray*}
        	\State Update inference network parameters $\theta_I \gets \theta_I + \alpha g_I$
        \EndWhile
		\State \textbf{Output:} $\mathbf{I}$
	\end{algorithmic}
	\caption{Training of the inference network in SCIGAN}\label{alg:inference}
\end{algorithm}

\clearpage

\section{Architecture} \label{app:arch}

\subsection{Definitions of Permutation Invariance and Permutation Equivariance} \label{app:inveqvdef}
The notions of what it means for a function to be {\em permutation invariant} and {\em permutation equivariant} with respect to (a subset of) its inputs are given below in definitions \ref{def:invar} and \ref{def:equi}, respectively. Let $\mathcal{U}$, $\mathcal{V}$, $\mathcal{C}$ be some spaces. Let $m \in \mathbb{Z}^+$.
\begin{definition} \label{def:invar}
A function $f : \mathcal{U}^m \times \mathcal{V} \to \mathcal{C}$ is permutation {\em invariant} with respect to the space $\mathcal{U}^m$ if for every $\mathbf{u} = (u_1, ..., u_m) \in \mathcal{U}^m$, every $v \in \mathcal{V}$ and every permutation, $\sigma$, of $\{1, ..., m\}$ we have \begin{equation}
    f(u_1, ..., u_m, v) = f(u_{\sigma(1)}, ..., u_{\sigma(m)}, v)\,.
\end{equation}
\end{definition}

\begin{definition} \label{def:equi}
A function $f : \mathcal{U}^m \times \mathcal{V} \to \mathcal{C}^m$ is permutation {\em equivariant} with respect to the space $\mathcal{U}^m$ if for every $\mathbf{u} \in \mathcal{U}^m$, every $v \in \mathcal{V}$ and every permutation, $\sigma$, of $\{1, ..., m\}$ we have $f(u_{\sigma(1)}, ..., u_{\sigma(m)}, v) = (f_{\sigma(1)}(\mathbf{u}, v), ..., f_{\sigma(m)}(\mathbf{u}, v))$, where $f_j(\mathbf{u}, v)$ is the $j$th element of $f(\mathbf{u}, v)$.
\end{definition}

To build up functions that are permutation invariant and permutation equivariant we make the following observations: (1) the composition of any function with a permutation invariant function is permutation invariant, (2) the composition of two permutation equivariant functions is permutation equivariant.

As noted in Section \ref{sec:hier_arch}, the basic building block we use for equivariant functions is defined in terms of equivariance input, $\mathbf{u}$, and auxiliary input, $\mathbf{v}$ by
\begin{equation} \label{eq:equifeat}
f_{equi}(\mathbf{u}, \mathbf{v}) = \sigma(\lambda \mathbf{I}_m \mathbf{u} + \gamma (\mathbf{1}_m \mathbf{1}_m^T)\mathbf{u} + (\mathbf{1}_m\Theta^T)\mathbf{v})\,.
\end{equation}

\clearpage

\section{Single Discriminator Model} \label{app:single}
In the paper we developed a hierarchical discriminator and demonstrated that it performs significantly better than the single discriminator setup that we now describe in this section.

\subsection{Single Discriminator}
In the single model, we will aim to learn a single discriminator, $\mathbf{D}$, that outputs $\mathbb{P}((W_f, D_f) = (w, d) | \mathbf{X},  \tilde{\mathcal{D}}_w, \tilde{\mathbf{Y}})$ for each $w \in \mathcal{W}$ and $d \in \tilde{\mathcal{D}}_w$. We will write $\mathbf{D}^{w, d}(\cdot)$ to denote the output of $\mathbf{D}$ that corresponds to the treatment-dosage pair $(w, d)$. We define the loss, $\mathcal{L}_D$, to be
\begin{equation}
\mathcal{L}_D(\mathbf{D}; \mathbf{G}) = - \mathbb{E} \Bigg[ \sum_{w \in \mathcal{W}} \sum_{d \in \tilde{\mathcal{D}}_w} \mathbb{I}_{\{T_f = (w, d)\}} \log \mathbf{D}^{w, d}(\mathbf{X}, \tilde{\mathbf{Y}}) + \mathbb{I}_{\{T_f \neq (w, d)\}} \log (1 - \mathbf{D}^{w, d}(\mathbf{X}, \tilde{\mathbf{Y}})) \Bigg]
\end{equation}
where the expectation is taken over $\mathbf{X}, \{\tilde{\mathcal{D}}_w\}_{w \in \mathcal{W}}, \tilde{\mathbf{Y}}, W_f$ and $D_f$ and we note that the dependence on $\mathbf{G}$ is through $\tilde{\mathbf{Y}}$. Our single discriminator will be trained to minimise this loss directly. The generator GAN-loss, $\mathcal{L}_G$, is then defined by
\begin{equation}
    \mathcal{L}_G(\mathbf{G}) = - \mathcal{L}_D(\mathbf{D}^*; \mathbf{G})
\end{equation}
where $\mathbf{D}^*$ is the optimal discriminator given by minimising $\mathcal{L}_D$. The generator will be trained to minimise $\mathcal{L}_G + \lambda \mathcal{L}_S$.

\subsection{Single Discriminator Architecture}

\begin{wrapfigure}{r}{0.45\columnwidth}
    \vspace{-0.5cm}
    \centering
    \includegraphics[height=7cm]{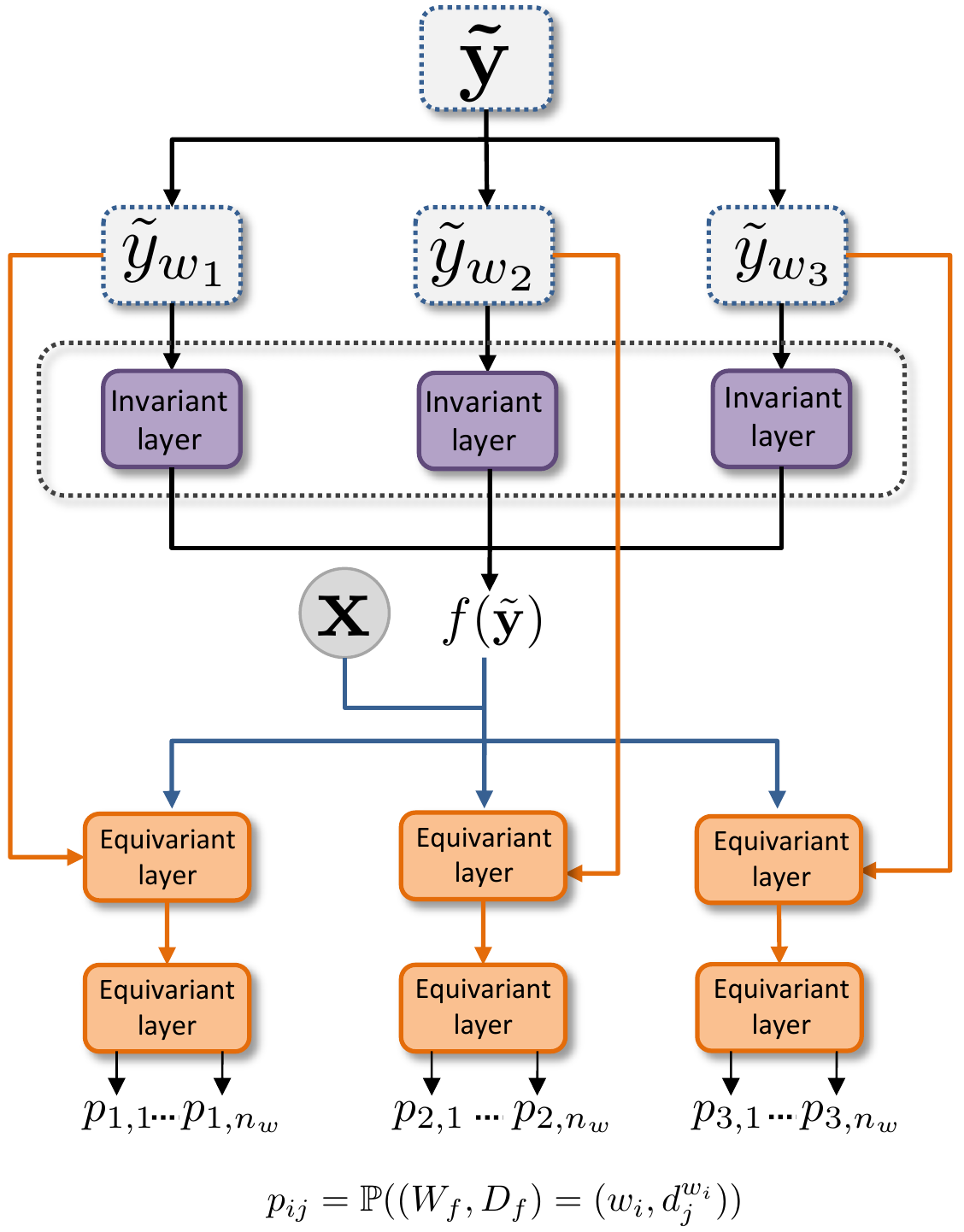}
    \caption{Overview of the single discriminator architecture.}
    \label{fig:sd_arch}
    \vspace{-0.1cm}
\end{wrapfigure}

In the case of the single discriminator, we want the output of $\mathbf{D}$ corresponding to each treatment $w \in \mathcal{W}$, i.e. $(\mathbf{D}^{w, 1}, ..., \mathbf{D}^{w, n_w})$, to be permutation equivariant with respect to $\tilde{\mathbf{y}}_w$ and permutation invariant with respect to each $\tilde{\mathbf{y}}_v$ for $v \in \mathcal{W} \setminus \{w\}$. To achieve this, we first define a function $f : \prod_{w \in \mathcal{W}} (\mathcal{D}_w \times \mathcal{Y})^{n_w} \to \mathcal{H}_S$ and require that this function be permutation invariant with respect to each of the spaces $(\mathcal{D}_w \times \mathcal{Y})^{n_w}$. For each treatment, $w \in \mathcal{W}$, we introduce a multitask head, $f_w : \mathcal{X} \times \mathcal{H}_S \times (\mathcal{D}_w \times \mathcal{Y})^{n_w} \to [0, 1]^{n_w}$, and require that each of these functions be permutation equivariant with respect to their corresponding input space $(\mathcal{D}_w \times \mathcal{Y})^{n_w}$ but they can depend on the features, $\mathbf{x} \in \mathcal{X}$, and invariant latent representation coming from $f$ arbitrarily. Writing $f_w^j$ to denote the $j$th output of $f_w$, the output of the discriminator given input features, $\mathbf{x}$, and generated outcomes, $\tilde{\mathbf{y}}$, is defined by
\begin{equation}
    \mathbf{D}^{w, j}(\mathbf{x}, \tilde{\mathbf{y}}) = f_w^i(\mathbf{x}, f(\tilde{\mathbf{y}}), \tilde{\mathbf{y}}_w).
\end{equation}

To construct the function $f$, we concatenate the outputs of several invariant layers of the form given in Eq. \ref{eq:equifeat} that each individually act on the spaces $(\mathcal{D}_w \times \mathcal{Y})^{n_w}$. That is, for each treatment, $w \in \mathcal{W}$ we define a map $f^w_{inv} : (\mathcal{D}_w \times \mathcal{Y})^{n_w} \to \mathcal{H}^w_S$ by substituting $\tilde{\mathbf{y}}_w$ for $\mathbf{u}$ in Eq. \ref{eq:equifeat}. We then define $\mathcal{H}_S = \prod_{w \in \mathcal{W}} \mathcal{H}_S^w$ and  $f(\tilde{\mathbf{y}}) = (f^{w_1}_{inv}(\tilde{\mathbf{y}}_{w_1}), ..., f^{w_k}_{inv}(\tilde{\mathbf{y}}_{w_k}))$.

Each $f_w$ will consist of two layers of the form given in Eq. (\ref{eq:equifeat}) with the equivariance input, $\mathbf{u}$, to first layer being $\tilde{\mathbf{y}}_w$ and to the second layer being the output of the first layer and the auxiliary input, $\mathbf{v}$, to the first layer being the concatenation of the features and invariant representation, i.e. $(\mathbf{x}, f(\tilde{\mathbf{y}}))$ and then no auxiliary input to the second layer.

A diagram depicting the architecture of the single discriminator model can be found in Fig. \ref{fig:sd_arch}.

\clearpage

\section{Ablation Studies architectures}
\label{app:ablation_arch}

\begin{figure}[h]
    \subfloat[Generator without multitask heads]{\includegraphics[height=5.5cm]{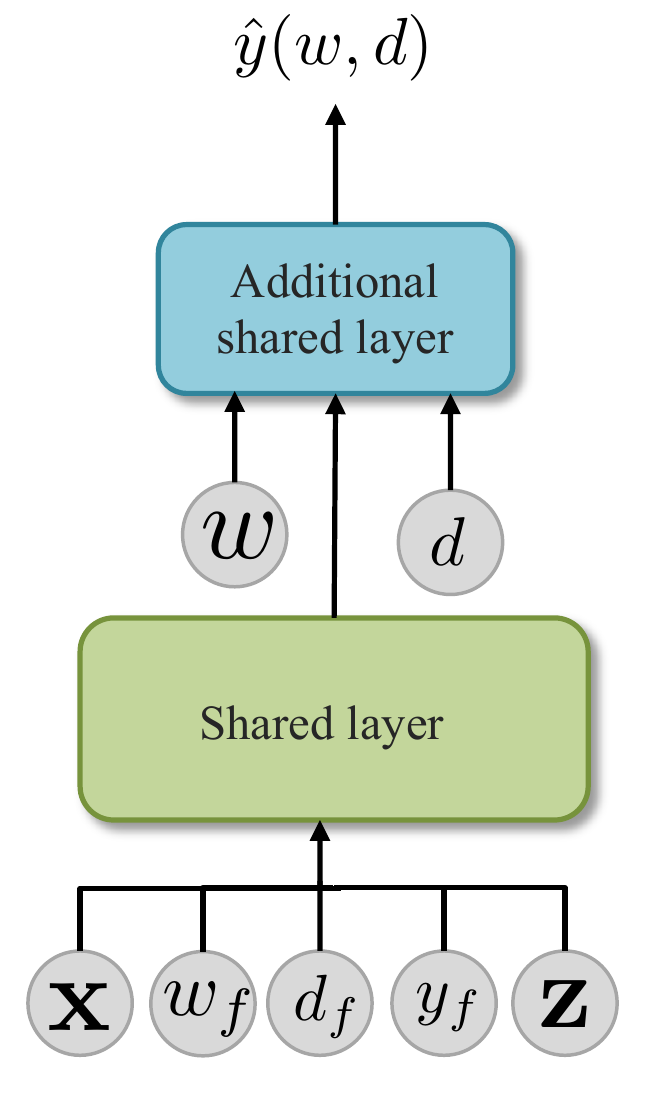} }
    \quad
	\subfloat[Treatment discriminator with fully connected (FC) layers instead of invariant layers]{\includegraphics[height=5.5cm]{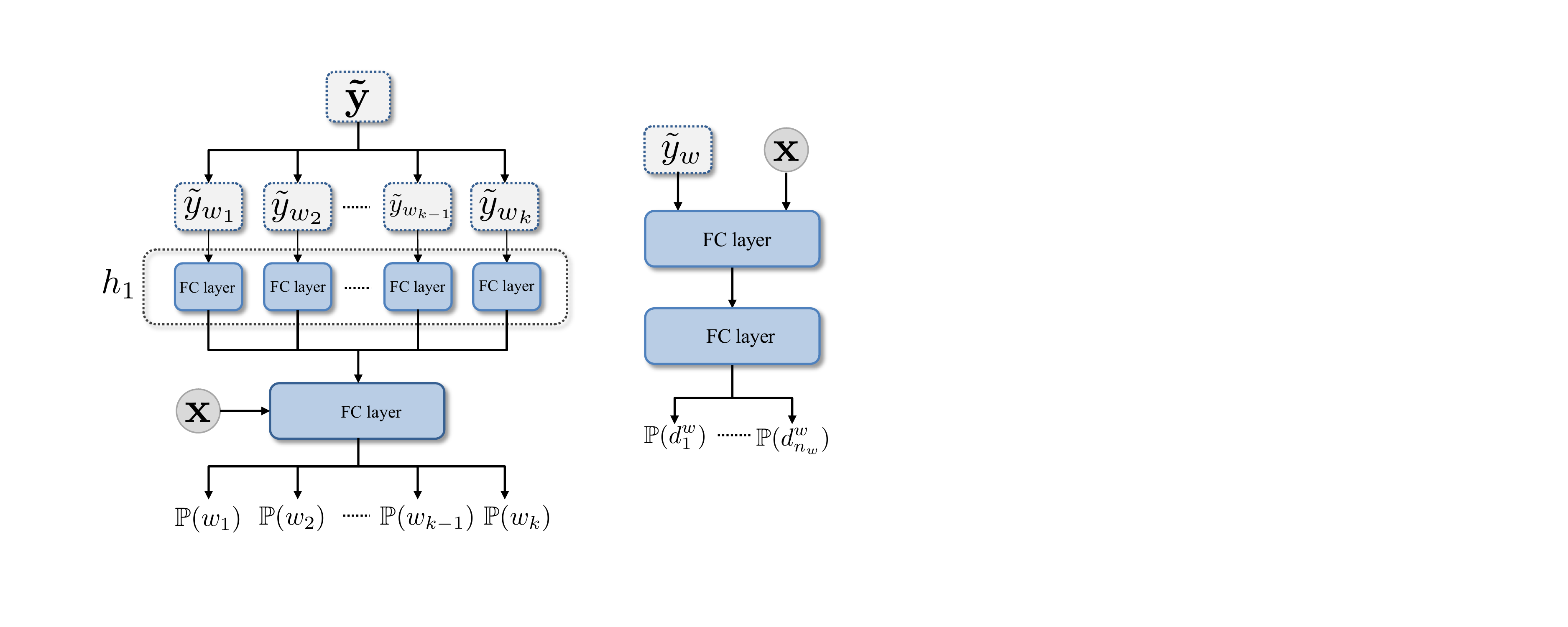} }
    \quad
    \subfloat[Dosage discriminator with FC layers instead of equivariant layers]{\includegraphics[height=5.5cm]{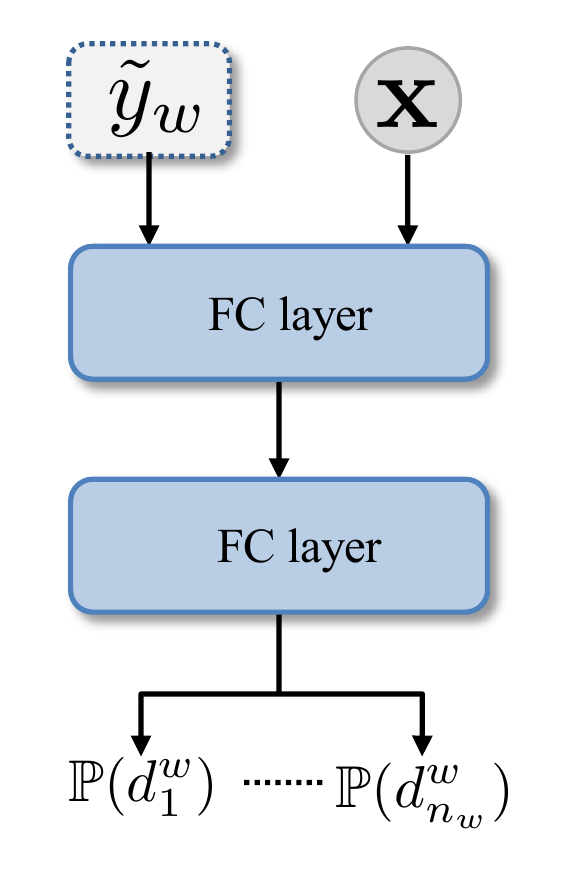}}%
    \caption{Architecture of the generator without multitask and discriminator without the invariant/equivariant layers used in the ablation studies.}
\vspace{-5mm}
\end{figure}

\section{Dataset descriptions} \label{app:data}
\textbf{TCGA:} The TCGA dataset consists of gene expression measurements for cancer patients \cite{weinstein2013cancer}. There are 9659 samples for which we used the measurements from the 4000 most variable genes. The gene expression data was log-normalized and each feature was scaled in the $[0, 1]$ interval. For each patient, the features were scaled to have norm $1$. We give meaning to our treatments and dosages by considering the treatment as being chemotherapy/radiotherapy/immunotherapy and their corresponding dosages. The outcome can be thought of as the risk of cancer recurrence \cite{schwab2019learning}. We used the same version of the TCGA dataset as used by DRNet \url{https://github.com/d909b/drnet}.

\textbf{News:} The News dataset consists of word counts for news items. We extracted 10000 news items (randomly sampled) each with 2858 features. As in \cite{johansson2016learning, schwab2019learning}, we give meaning to our treatments and dosages by considering the treatment as being the viewing device (e.g. phone, tablet etc.) used to read the article and the dosage as being the amount of time spent reading it. The outcome can be thought of as user satisfaction. We used the same version of the News dataset as used by DRNet \url{https://github.com/d909b/drnet}. 

\textbf{MIMIC III:} The Medical Information Mart for Intensive Care (MIMIC III) \cite{mimiciii} database consists of observational data from patients in the ICU. We extracted 3000 patients (randomly sampled) that receive antibiotics treatment and we used as features 9 clinical covariates, namely age, temperature, heart rate, systolic and diastolic blood pressure, SpO2, FiO2, glucose, and white blood cell count, measured at start of ICU stay. Again, the features were scaled in the $[0, 1]$ interval. In this setting, we can considered as treatments the different antibiotics and their corresponding dosages. 

For a summary description of the datasets, see table \ref{tab:datasets}. The datasets are split into 64/16/20\% for training, validation and testing respectively. The validation dataset is used for hyperparameter optimization.

\begin{table*}[h!]
    	\centering
    	\begin{small}
    	\begin{tabular}{lccc}
    		\toprule
    		\multirow{2}{*}{\textbf{}} & \multicolumn{1}{c}{\textbf{TCGA}} &
    		\multicolumn{1}{c}{\textbf{News}} & \multicolumn{1}{c}{\textbf{MIMIC}} \\
    		\midrule
    		Number of samples & 9659& 10000 & 3000 \\
    		Number of features & 4000  & 2858 & 9 \\
    		Number of treatments &  3* & 3 & 2   \\
    		\bottomrule
    	\end{tabular}
    	\end{small}
    	\caption{Summary description of datasets. *: for our final experiment in Appendix \ref{app:add_num_treat} we increase the number of treatments in TCGA to 6 and 9.}    
    	\label{tab:datasets}
    \end{table*}

\newpage
\section{Dosage bias} \label{app:dosass}
In order to create dosage-assignment bias in our datasets, we assign dosages according to $d_w | \mathbf{x} \sim$ Beta$(\alpha, \beta_w)$. The selection bias is controlled by the parameter $\alpha \geq 1$. When we set $\beta_w = \frac{\alpha - 1}{d^*_w} + 2 - \alpha$ (which ensures that the mode of our distribution is $d_w^*$), we can write the variance of $d_w$ in terms of $\alpha$ and $d^*_w$ as follows
\begin{equation}
    \text{Var}(d_w) = \frac{\frac{\alpha^2 - \alpha}{d^*_w} + 2\alpha - \alpha^2}{(\frac{\alpha - 1}{d^*_w} + 2)^2(\frac{\alpha - 1}{d^*_w} + 3)} \approx \frac{c\alpha^2}{d\alpha^3}\,.
\end{equation}
We see that the variance of our Beta distribution therefore decreases with $\alpha$, resulting in the sampled dosages being closer to the optimal dosage, thus resulting in higher dosage-selection bias. In addition we note that the Beta$(1, 1)$ distribution is in fact the uniform distribution, corresponding to the dosages being sampled independently of the patient features, resulting in no selection bias when $\alpha = 1$.

\clearpage

\section{Benchmarks} \label{app:benchmarks}
We use the publicly available GitHub implementation of DRNet provided by \cite{schwab2019learning}: \url{https://github.com/d909b/drnet}. Moreover, we also used a GPS implementation similar to the one from  \url{https://github.com/d909b/drnet} which uses the \texttt{causaldrf} R package \cite{galagate2016causal}. More spcifically, the GPS implementation uses a normal treatment model, a linear treatment formula and a $2$-nd degree polynomial for the outcome. Moreover, for the TCGA and News datasets, we performed PCA and only used the 50 principal components as input to the GPS model to reduce computational complexity.

\textbf{Hyperparameter optimization:} The validation split of the dataset is used for hyperparameter optimization. For the DRNet benchmarks we use the same hyperparameter optimization proposed by \cite{schwab2019learning} with the hyperparameter search ranges described in Table \ref{tab:hyperparameters_drnet}. For SCIGAN, we use the hyperparameter optimization method proposed in GANITE \cite{yoon2018ganite}, where we use the complete dataset from the counterfactual generator to evaluate the MISE on the inference network. We perform a random search \cite{bergstra2012random} for hyperparameter optimization over the search ranges in Table \ref{tab:hyperparameters_ganite}. For all experiments with SCIGAN, we used
5000 training iterations for the GAN network and 10000 training iterations for the inference network. This number of training iterations was chosen to ensure convergence
of the generator loss, discriminator loss, as well as of the supervised loss. For a fair comparison, for the MLP-M model we used the same architecture used in the inference network of SCIGAN. Similarly, for the MLP model we use the same architecture as for the MLP-M, but without the multitask heads. 

\begin{table*}[h!]
    	\centering
    	\begin{small}
    	\begin{tabular}{lc}
    		\toprule
    		{\textbf{Hyperparameter}} & {Search range} \\
    		\midrule
    		Batch size &  32, 64, 128 \\
    		Number of units per hidden layer & 24, 48, 96, 192 \\
    		Number of hidden layers & 2, 3  \\
    		Dropout percentage  & 0.0, 0.2   \\
    		Imbalance penalty weight$^{*}$ & 0.1, 1.0, 10.0 \\ 
    		\midrule
    		& Fixed \\
    		\midrule
    		Number of dosage strata $E$ & 5 \\
    		\bottomrule
    	\end{tabular}
    	\end{small}
    	\caption{Hyperparameters search range for DRNet. *: For the DRNet model using Wasserstein regularization only.}    
    	\label{tab:hyperparameters_drnet}
\end{table*}

\begin{table*}[h!]
    	\centering
    	\begin{small}
    	\begin{tabular}{lc}
    		\toprule
    		{\textbf{Hyperparameter}} & {Search range}  \\
    		\midrule
    		Batch size &  64, 128, 256 \\
    		Number of units per hidden layer & 32, 64, 128 \\
    		Size of invariant and equivariant representations & 16, 32, 64, 128 \\
    		\midrule
    		& Fixed \\
    		\midrule
    		Number of hidden layers per multitask head & 2 \\
    		Number of dosage samples & 5 \\
    		$\lambda$ & 1 \\
    		Optimization & Adam Moment Optimization \\
    		\bottomrule
    	\end{tabular}
    	\end{small}
    	\caption{Hyperparameters search range for SCIGAN.}    
    	\label{tab:hyperparameters_ganite}
\end{table*}

The hyperparameters used to generate the results for SCIGAN are given in Table \ref{tab:hyperparameters_results_scigan}. 

\begin{table*}[h!]
    	\centering
    	\begin{small}
    	\begin{tabular}{lccc}
    		\toprule
    		{\textbf{Hyperparameter}} & {TCGA} & {News} & {MIMIC} \\
    		\midrule
    		Batch size & 128 & 256 & 128  \\
    		Number of units per hidden layer & 64 &  128 & 32 \\
    		Size of invariant and equivariant representations & 16 & 32 & 16  \\
    		Number of hidden layers per multitask head & 2 & 2 & 2 \\
    		Number of dosage samples & 5 & 5 & 5\\
    		$\lambda$ & 1 & 1 & 1 \\
    		\bottomrule
    	\end{tabular}
    	\end{small}
    	\caption{Hyperparameters used for obtaining results.}    
    	\label{tab:hyperparameters_results_scigan}
\end{table*}

The experiments were run on a system with 6CPUs, an Nvidia K80 Tesla GPU and 56GB of RAM.

\clearpage

\section{Metrics} \label{app:metrics}
The Mean Integrated Square Error (MISE) measures how well the models estimates the patient outcome across the entire dosage space:
\begin{align}
\text{MISE} = \frac{1}{N} \frac{1}{k} \sum_{w\in \mathcal{W}}\sum_{i=1}^{N} \int_{\mathcal{D}_w} \Big( y^{i}(w, u) - \hat{y}^{i}(w, u) \Big)^2 \text{d}u \,.
\end{align}

In addition to this, we also compute the mean dosage policy error (DPE) \cite{schwab2019learning} to assess the ability of the model to estimate the optimal dosage point for every treatment for each individual:
\begin{align}
\text{DPE} = \frac{1}{N} \frac{1}{k} \sum_{w\in \mathcal{W}}\sum_{i=1}^{N} \Big( y^{i}(w, d_w^{*}) - {y}^{i}(w, \hat{d}_w^{*}) \Big)^2\,, 
\end{align}
where $d_w^{*}$ is the true optimal dosage and $\hat{d}_w^{*}$ is the optimal dosage identified by the model. The optimal dosage points for a model are computed using SciPy's implementation of Sequential Least SQuares Programming.

Finally, we compute the mean policy error (PE) \cite{schwab2019learning} which compares the outcome of the true optimal treatment-dosage pair to the outcome of the optimal treatment-dosage pair as selected by the model:
\begin{align}
\text{PE} = \frac{1}{N} \sum_{i=1}^{N} \Big( y^{i}(w^*, d_w^{*}) - {y}^{i}(\hat{w}^*, \hat{d}_w^{*}) \Big)^2\,, 
\end{align}
where $w^*$ is the true optimal treatment and $\hat{w}^*$ is the optimal treatment identified by the model. The optimal treatment-dosage pair for a model is selected by first computing the optimal dosage for each treatment and then selecting the treatment with the best outcome for its optimal dosage.

Each of these metrics are computed on a held out test-set.

\clearpage

\section{Additional results} \label{app:addresults}
\subsection{News results for source of gain}
\begin{table*}[h]
        \centering
        \begin{small}
        \setlength\tabcolsep{4.4pt}
        \begin{tabular}{lccc}
        \toprule
        \multirow{2}{*}{}& \multicolumn{3}{c}{\textbf{News}} \\
        & $\sqrt{\text{MISE}}$ & $\sqrt{\text{DPE}}$ & $\sqrt{\text{PE}}$ \\
        \midrule
         Baseline & $6.17 \pm 0.27$ & $6.97 \pm 0.27$ & $6.20 \pm 0.21$\\
        \cmidrule{1-4}
        {+ $\mathcal{L}_S$} & $4.51 \pm 0.16$ & $4.46 \pm 0.12$ & $4.40 \pm 0.11$\\
        \cmidrule{1-4}
        {+ Multitask} &  $4.11 \pm 0.11$ & $4.33 \pm 0.11$ & $4.31 \pm 0.11$\\
        \cmidrule{1-4}
        {+ Hierarchical} & $4.07 \pm 0.05$  & $4.24 \pm 0.11$ & $4.17 \pm 0.12$ \\
        \cmidrule{1-4}
        {+ Inv/Eqv} & $3.71 \pm 0.05$ & $4.14 \pm 0.11$ & $3.90 \pm 0.05$ \\
        \bottomrule
        \end{tabular}
        \end{small}
        \caption{Source of gain analysis for our model on the News dataset. Metrics are reported as Mean $\pm$ Std.
        \label{tab:source_of_gain_news}}
\end{table*}

\subsection{Investigating hyperparameter sensitivity ($n_w$)} \label{app:add_hype}
The performance of the single discriminator causes significant performance drops around $n_w = 9$ across all metrics. As previously noted, this is due to the dimension of the output space (which for $n_w = 9$ is $27$) being too large. Conversely, we see that our hierarchical discriminator shows much more stable performance even when $n_w = 19$. 

Here we present additional results for our investigation of the hyperparameters $n_w$. Fig. \ref{fig:num_dosage_samples_news} reports each of the 3 performance metrics as we increase the number of dosage samples, $n_w$, used to train the discriminators on the News dataset. As with the TCGA results in the main paper we see that the single discriminator suffers a significant performance decrease when $n_w$ is set too high.

\begin{figure}[H]
    \centering
    \subfloat[$\sqrt{\text{MISE}}$]{\includegraphics[height=3.2cm]{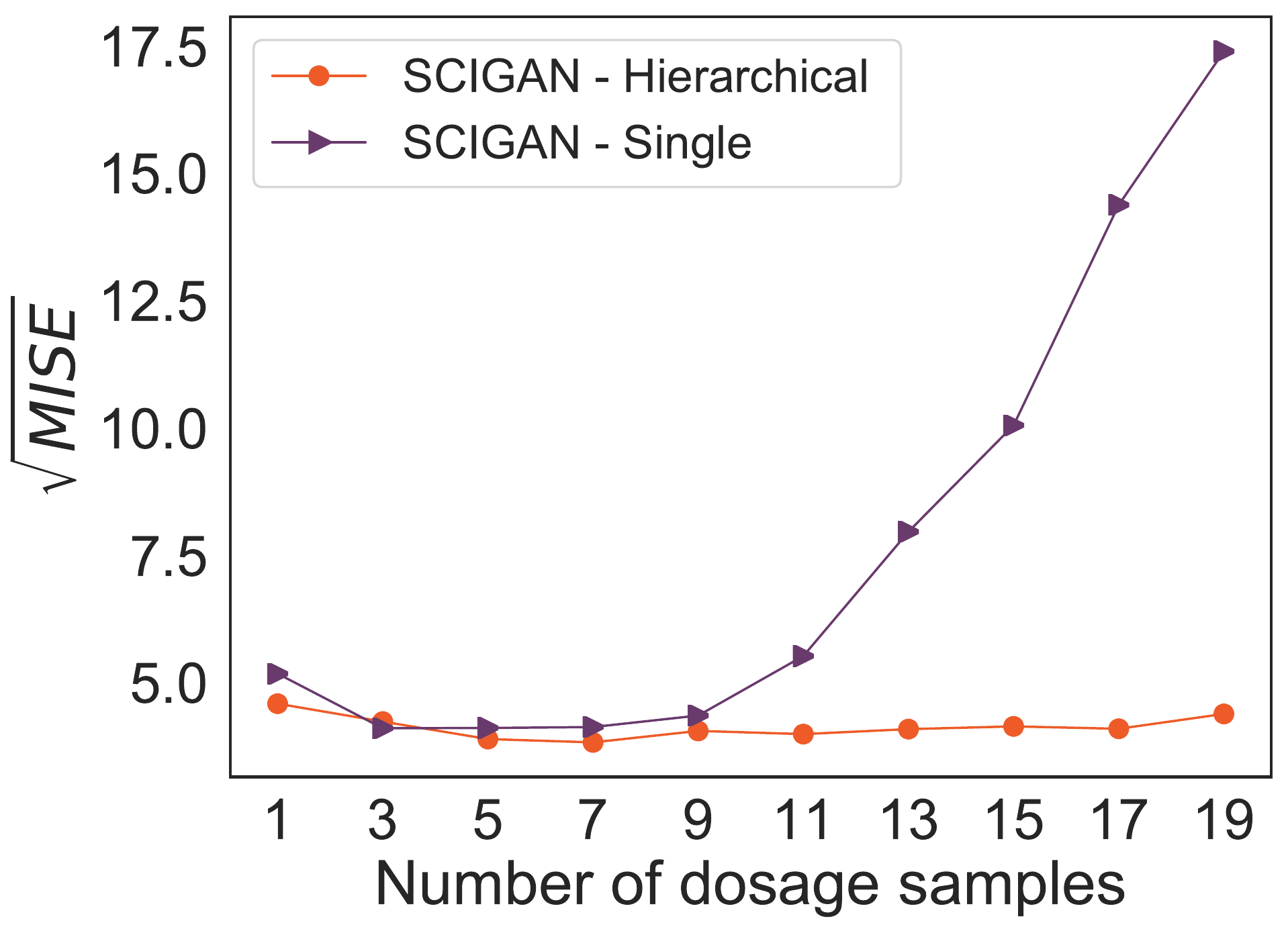} }
    \quad
    \subfloat[$\sqrt{\text{DPE}}$]{\includegraphics[height=3.2cm]{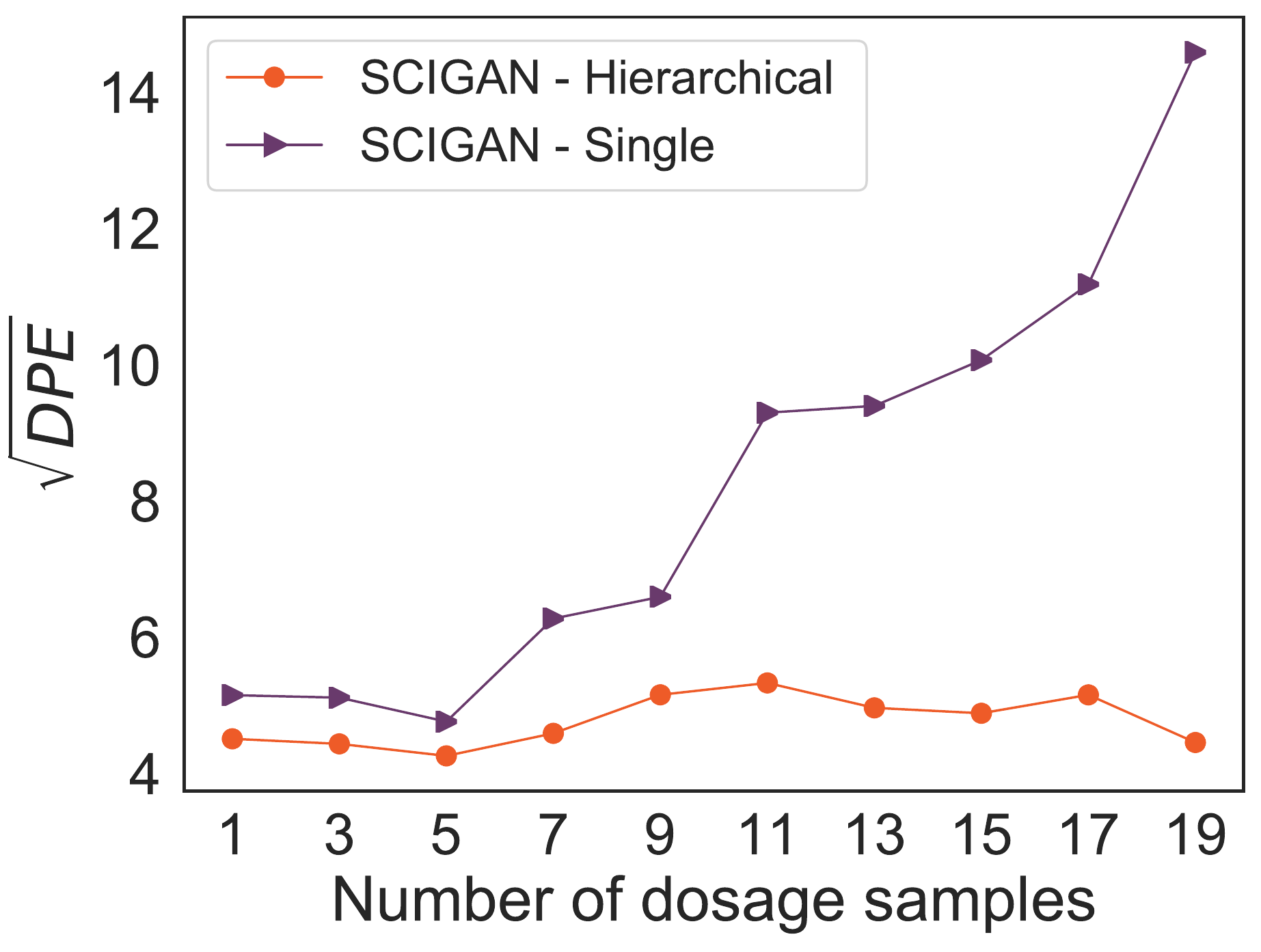}}
    \quad
    \subfloat[$\sqrt{\text{PE}}$]{\includegraphics[height=3.2cm]{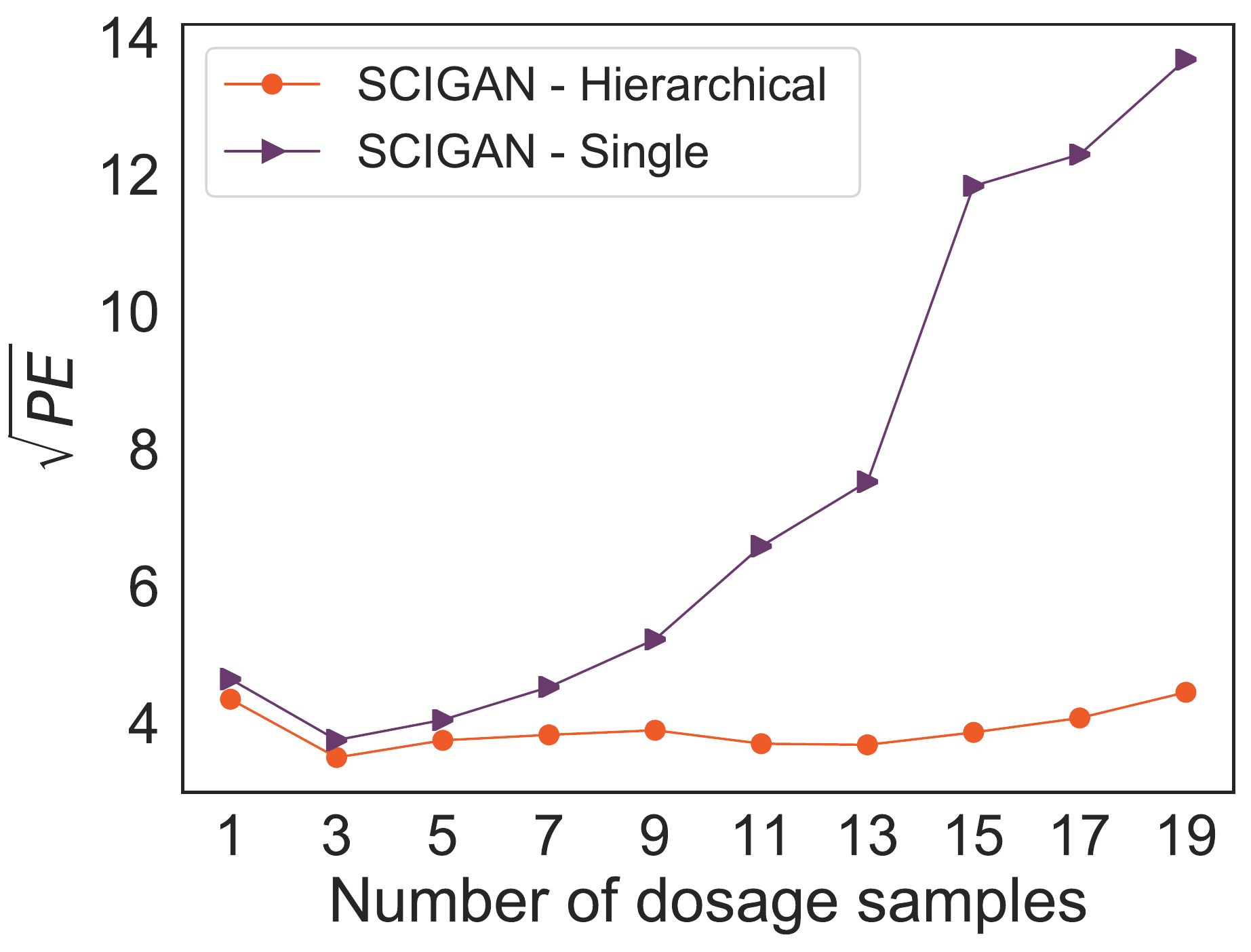}}
    \caption{Performance of single vs. hierarchical discriminator when increasing the number of dosage samples ($n_w$) on News dataset.}
    \label{fig:num_dosage_samples_news}
\end{figure}

\begin{figure*}[h]
    \centering
    \subfloat[$\sqrt{\text{MISE}}$]{\includegraphics[height=3.2cm]{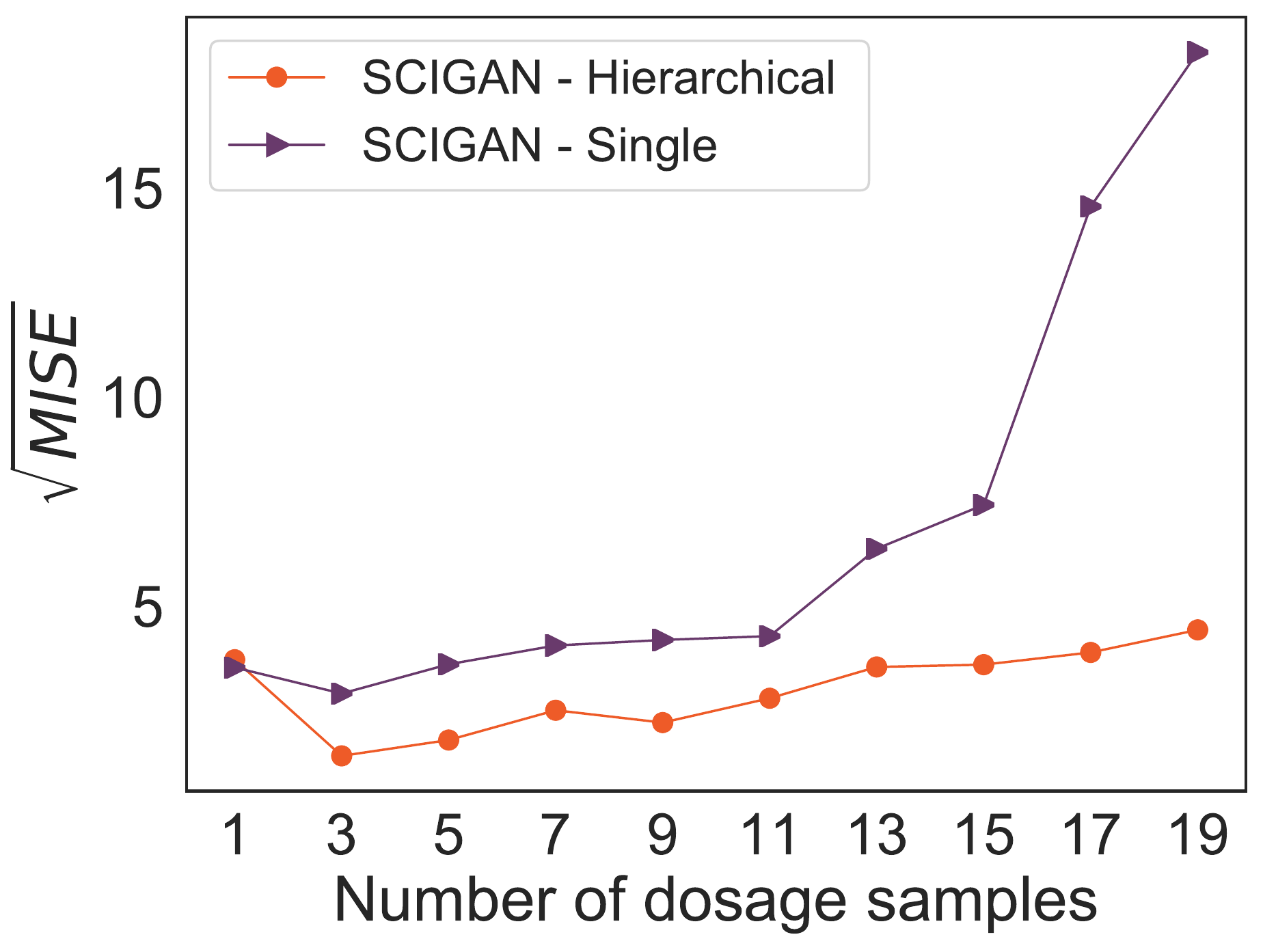} }
    \quad
    \subfloat[$\sqrt{\text{DPE}}$]{\includegraphics[height=3.2cm]{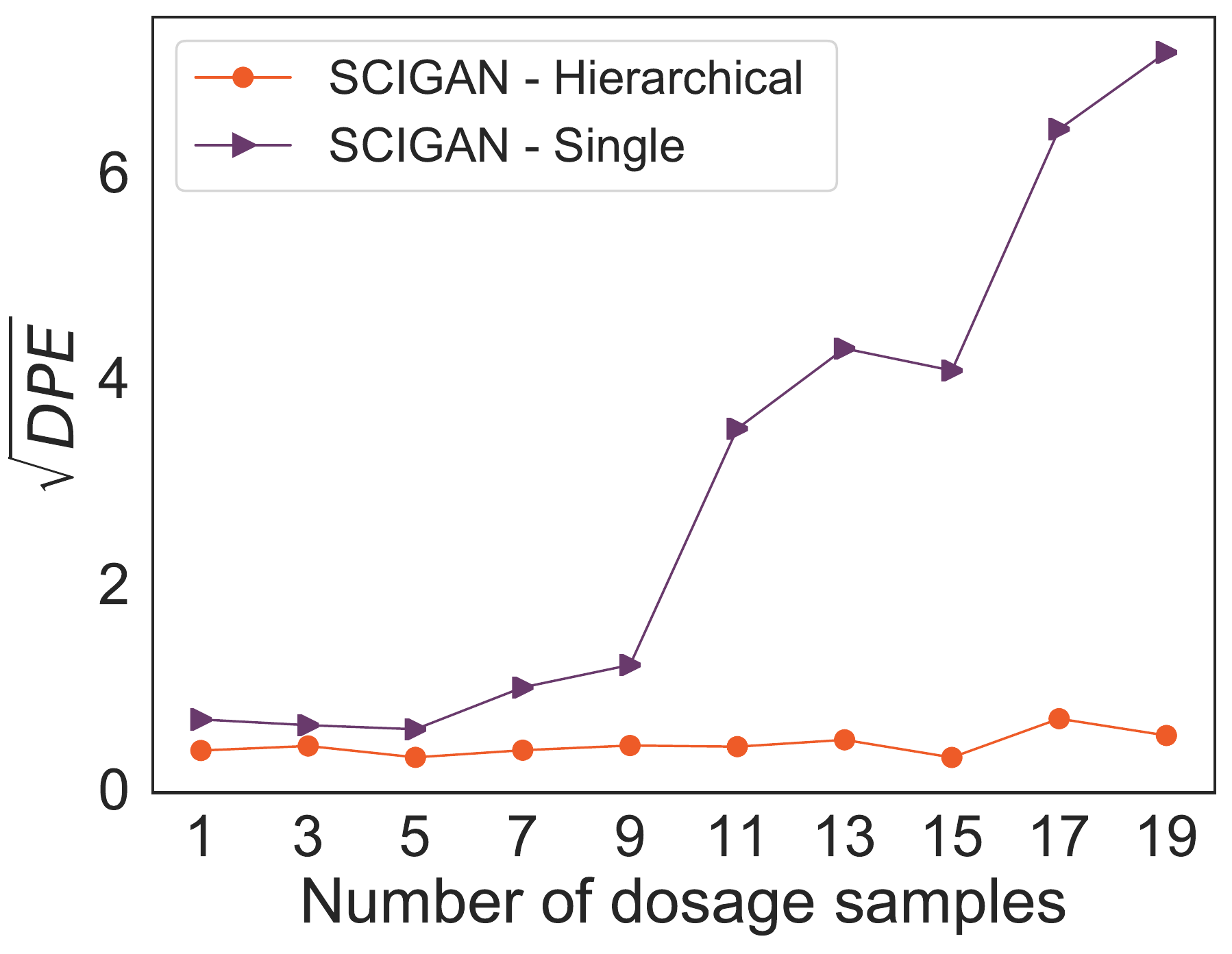}}
    \quad
    \subfloat[$\sqrt{\text{PE}}$]{\includegraphics[height=3.2cm]{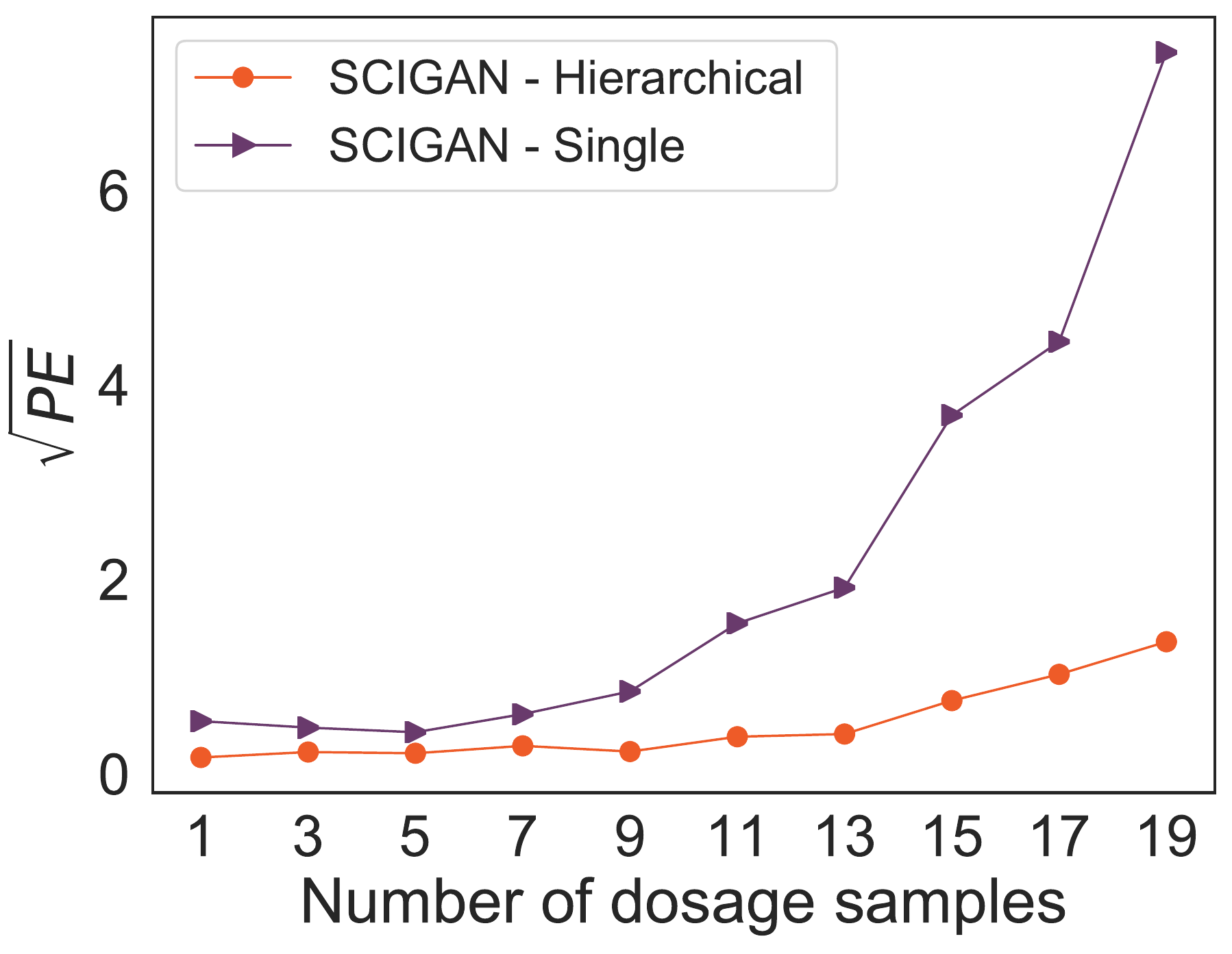}}
    \caption{Performance of single vs. hierarchical discriminator when increasing the number of dosage samples ($n_w$) on TCGA dataset.}
    \label{fig:num_dosage_samples_tcga}
\end{figure*}

\clearpage

\subsection{Dosage Policy Error for Benchmark Comparison} \label{app:add_benchmark}
In Table \ref{tab:benchmarks_dpe} we report the Dosage Policy Error (DPE) corresponding to Section \ref{sec:bench_results} in the main paper.

\begin{table*}[h!]
    	\centering
    	\begin{small}
    	\begin{tabular}{lccc}
    		\toprule
    		\multirow{2}{*}{\textbf{Methods}} & \multicolumn{1}{c}{\textbf{TCGA}} &
    		\multicolumn{1}{c}{\textbf{News}} & \multicolumn{1}{c}{\textbf{MIMIC}} \\
    		& \multicolumn{1}{c}{$\sqrt{\text{DPE}}$} & \multicolumn{1}{c}{$\sqrt{\text{DPE}}$} &  \multicolumn{1}{c}{$\sqrt{\text{DPE}}$} \\
    		\midrule
    		\textbf{SCIGAN} & $\mathbf{0.31} \pm 0.05$ & $\mathbf{4.14} \pm 0.11$ & $\mathbf{0.51} \pm 0.05$ \\
    		\midrule
    		DRNet & $0.51 \pm 0.05^{*} $ & $4.39 \pm {0.11}^{*} $ & $0.52 \pm 0.05$ \\
    		DRN-W &  $0.50 \pm 0.05^{*} $ & $4.21\pm 0.11 $ & $0.53 \pm 0.05$    \\
    		GPS & $1.38 \pm 0.01^{*} $ & $6.40 \pm 0.01^{*} $ & $1.41 \pm 0.12^{*}$    \\
    		\midrule
    		MLP-M & $0.92 \pm 0.05^{*}  $ & $4.94 \pm 0.16^{*} $ & $0.77\pm 0.05^{*}$  \\
    		MLP & $1.04 \pm 0.05^{*}  $ & $5.18 \pm 0.12^{*} $ & $0.80 \pm 0.05^{*}$ \\

    		\bottomrule
    	\end{tabular}
    	\end{small}
    	\caption{Performance of individualized treatment-dose response estimation on three datasets.  Bold indicates the method with the best performance for each dataset. *: performance improvement is statistically significant.}    
    	\label{tab:benchmarks_dpe}
    \end{table*}

\subsection{Varying the number of treatments} \label{app:add_num_treat}
In this experiment, we increase the number of treatments by defining 3 or 6 additional treatments. The parameters $\mathbf{v}^w_1, \mathbf{v}^w_2, \mathbf{v}^w_3$ are defined in exactly the same way as for 3 treatments. The outcome shapes for treatments 4 and 7 are the same as for treatment 1, similarly for 5, 8 and 2 and for 6, 9 and 3. In Table \ref{tab:num_treat} we report MISE, DPE and PE on the TCGA dataset with 6 treatments (TCGA-6) and with 9 treatments (TCGA-9). Note that we use 3 dosage samples for training SCIGAN in this experiment. 

\begin{table*}[h]
        \centering
        \begin{small}
        \begin{tabular}{lcccccc}
        \toprule
        \multirow{2}{*}{Method}&\multicolumn{3}{c}{\textbf{TCGA - 6}} & \multicolumn{3}{c}{\textbf{TCGA - 9}} \\
        & $\sqrt{\text{MISE}}$ & $\sqrt{\text{DPE}}$ & $\sqrt{\text{PE}}$ & $\sqrt{\text{MISE}}$ & $\sqrt{\text{DPE}}$ & $\sqrt{\text{PE}}$ \\
        \midrule
         SCIGAN & $2.37 \pm 0.12$ & $0.43 \pm 0.05$ & $0.32 \pm 0.05$ & $2.79 \pm 0.05$ & $0.51 \pm 0.05$ & $0.54 \pm 0.05$\\
        \midrule
        {DRNET} & $4.09 \pm 0.16$ & $0.52 \pm 0.05$ &  $0.71 \pm 0.05$ & $4.31 \pm 0.12$ & $0.59 \pm 0.05$ & $0.74 \pm 0.05$\\
        {GPS} & $6.62\pm 0.01$  & $2.04 \pm 0.01$ & $2.61 \pm 0.00$ & $7.58 \pm 0.01$ & $3.14 \pm 0.01$ & $2.91 \pm 0.01$\\
        \bottomrule
        \end{tabular}
        \end{small}
        \caption{Performance of SCIGAN and the benchmarks when we increase the number of treatments in the dataset to 6 and 9. Bold indicates the method with the best performance for each dataset.}
        \label{tab:num_treat}
    \end{table*}
    
\subsection{Sample efficiency}
We have also performed a further experiment to evaluate model performance in terms of sample efficency. For the MIMIC dataset, in Table 1 we report evaluation metrics for training SCIGAN with different number of training samples $N$ and evaluating on the same test set.

\begin{table*}[h]
        \centering
        \setlength\tabcolsep{1.4pt}
        \begin{tabular}{lccc}
        \toprule
        & $\sqrt{\text{MISE}}$ & $\sqrt{\text{DPE}}$ & $\sqrt{\text{PE}}$ \\
        \midrule
        $N=100$ & $31.12 \pm 63.39$ & $7.72 
        \pm 2.57$ & $18.94 \pm 29.07$ \\
         $N=500$ & $13.36 \pm 10.46$ & $4.07
         \pm 1.92$ & $2.63 \pm 0.94$ \\
        $N=1000$ & $3.80 \pm 1.04$ & $2.46 \pm 1.75$ & $1.03 \pm 1.13$ \\
        $N=1500$ & $2.95 
        \pm 0.37$ & $0.70 \pm 0.17$ & $0.63 \pm 0.12$ \\
        $N=1920$ & $2.09 \pm 0.12$ & $0.51 \pm 0.05$ &  $0.32 \pm 0.05$  \\
        \bottomrule
        \end{tabular}
        \caption{Sample efficency analysis for MIMIC. Metrics are reported as Mean $\pm$ Std.
        }
        \label{tab:sample_efficency}
\end{table*}
    
\clearpage

\subsection{Discrete dosage set-up and comparison with GANITE} \label{app:disc_dos}

In this set-up, we use the TCGA dataset and treatments 2 and 3 from Table \ref{tab:data_simultation} with dose-response curves $f_2(\mathbf{x}, d)$ and $f_3(\mathbf{x}, d)$ respectively. Let $\beta$ be the number of discrete dosages for which we want to generate data. We chose $\beta$ equally spaced points in the interval $[0, 1]$ as our set of discrete dosages: $\Delta = \{\frac{k}{\beta-1}\}_{k=0}^{\beta-1}$. To create factual dosages for our dataset, we sample the dosages as before $d_w \mid x \sim \text{Beta}(\alpha, \beta_w)$ (as described in Appendix 
\ref{app:dosass}), and choose the closest discrete dosage from the set $\Delta$. 

To evaluate SCIGAN in this setting, we maintain the same architecture for the multi-task generator and hierarchical discriminator. The only difference is that we now randomly sample dosages for the SCIGAN discriminator from $\Delta$. 

We adopt the GANITE implementation proposed by \cite{yoon2018ganite}. To be able to have a fair comparison with SCIGAN we also use a multi-task architecture for the GANITE generator and we give as input to each multitask head the dosage parameter. The GANITE generator will generate outcomes for all possible discrete dosages in $\Delta$ and these will be passed to the GANITE discriminator to distinguish the factual one. For the GANITE generator we use a similar architecture to the SCIGAN generator with 2 hidden layers for each multitask head and 64 neurons in each layer. The GANITE discriminator consists of 2 fully connected layers with 64 neurons in each. We also set $\lambda = 1$. In addition, to maintain a similar set-up to SCIGAN, we train an inference network to learn the counterfactual outcomes with data from the GANITE generator. The inference network has the same architecture as the GANITE generator.

\begin{figure*}[h]
    \centering
    \subfloat[$\sqrt{\text{DPE}}$]{\includegraphics[height=4.5cm]{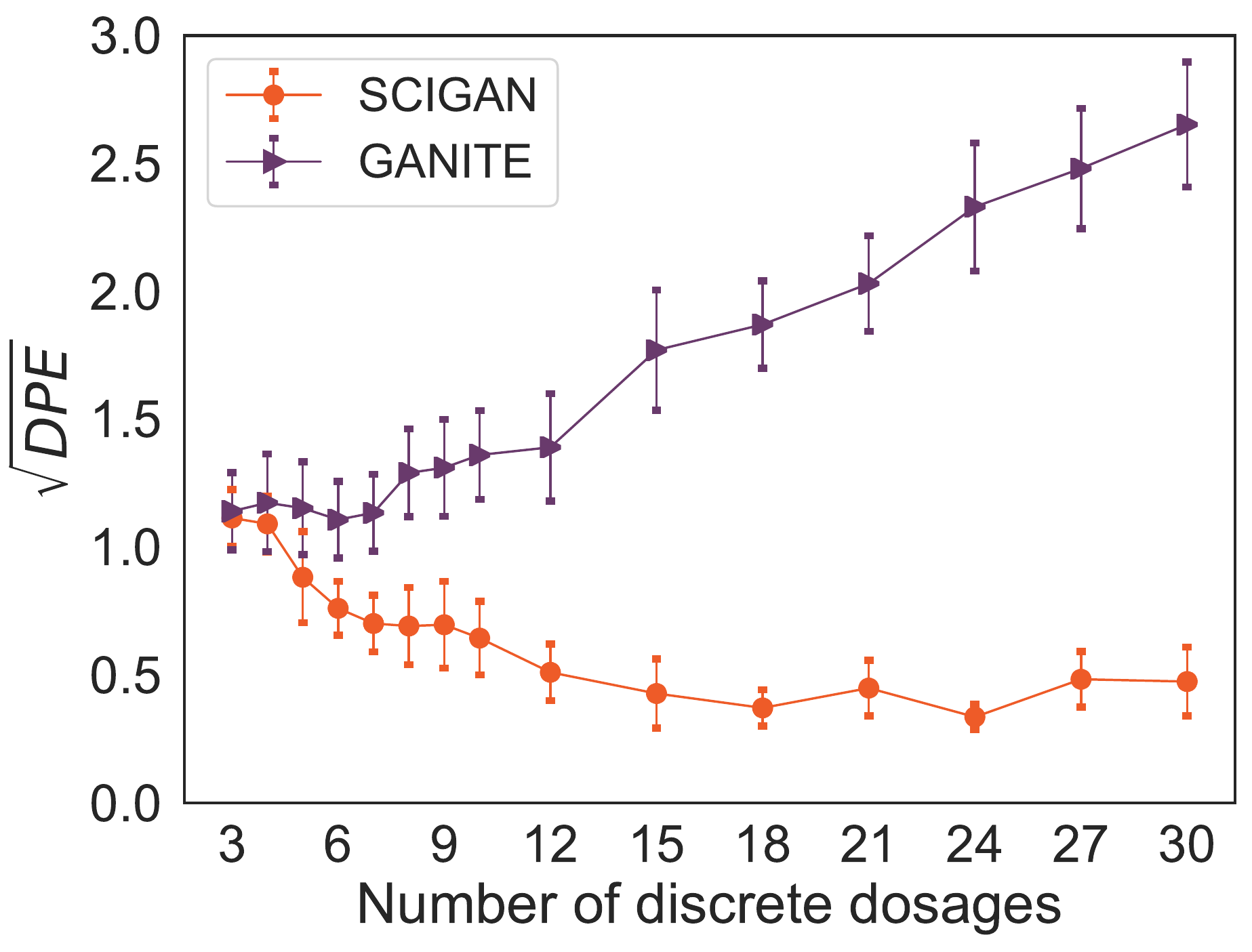} }
    \quad
    \subfloat[$\sqrt{\text{PE}}$]{\includegraphics[height=4.5cm]{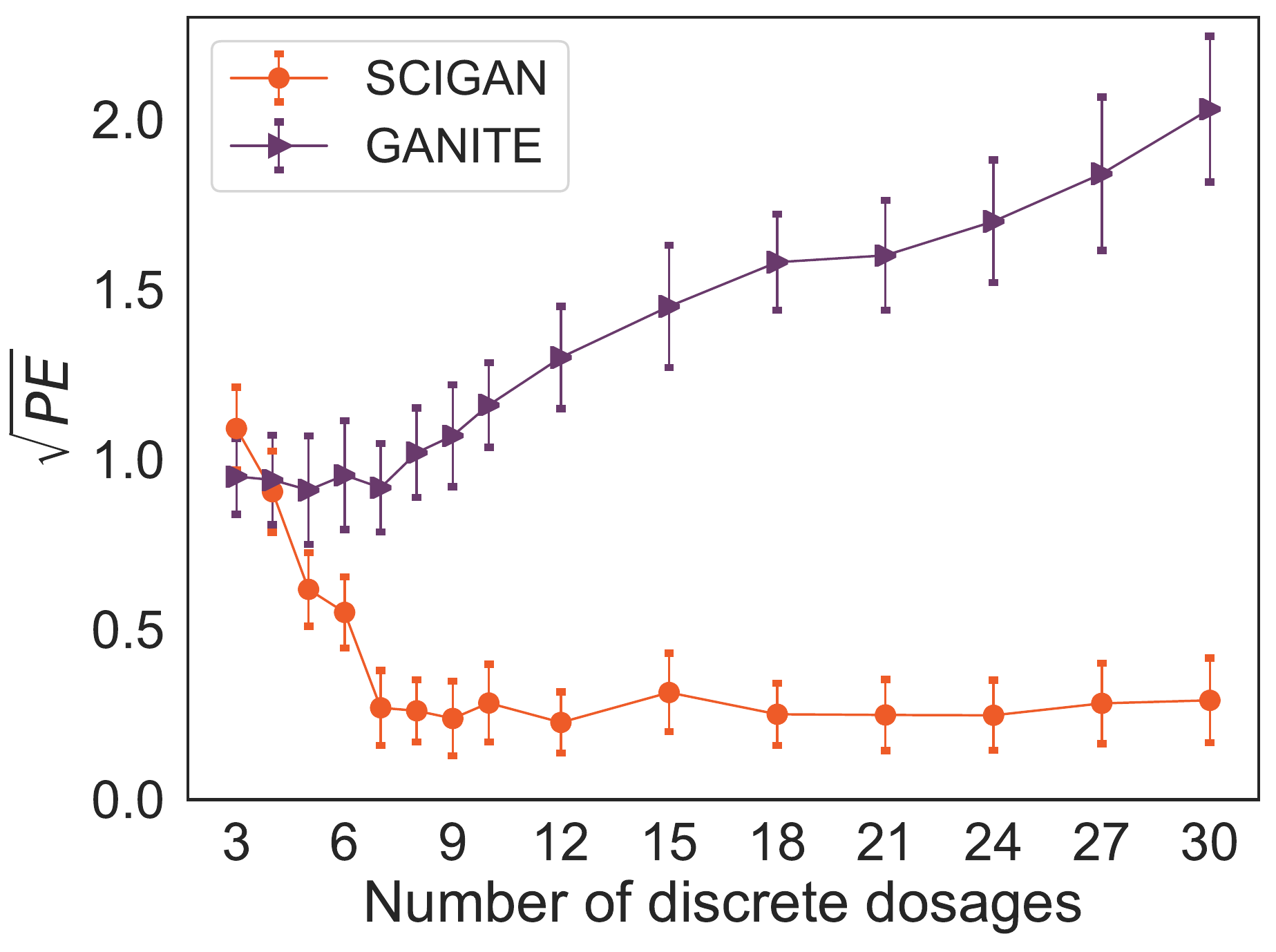}}
    \caption{Comparison between SCIGAN and GANITE in the discrete dosage set-up.}
    \label{fig:dpe_pe_discrete_tcga}
\end{figure*}

We clearly see from Fig. \ref{fig:dpe_pe_discrete_tcga} that SCIGAN achieves a similar performance to GANITE for a small number of dosages ($<6$) but then significantly outperforms GANITE for more dosages than $6$. In fact, we see that while GANITE's performance degrades with an increasing number of dosages, SCIGAN's improves and then stabilises at around 12 dosages. This is due to the fact that the single discriminator in GANITE simply cannot handle a large number of dosages. Our hierarchical model, however, can. The worse performance of SCIGAN for the lower dosages can be attributed to the fact that for such few dosages (e.g. 3 dosages corresponds to only 6 total different interventions), the SCIGAN architecture is overly complex, and the sub-sampling of dosages for the discriminator is not actually necessary.

\clearpage

\subsection{Mixing dosage and no-dosage treatment options}

We also evaluate the case when one of the treatments does not have a dosage parameter. For this experiment we generate data for the treatment that has a dosage parameter $d$ using $f_3(\mathbf{x}, d)$ and for the treatment without an associated dosage using $2C(\mathbf{v}_0^{T}\mathbf{x})$, where $\mathbf{v}_0$ are parameters, $\mathbf{x}$ are patient features and $C=$ is the scaling parameter. This set-up also corresponds to the scenario where we want to compare giving a treatment with a dosage and not giving any treatment.

SCIGAN can be easily extended to incorporate an additional treatment that does not come with a dosage parameter. Such treatments will not need a dosage discriminator but will be passed to the treatment discriminator. A head can be added to the generator for each such non-dosage treatment but will not need to take dosage as an input.

As the DRNet public implementation does not allow for this set-up, we compared SCIGAN with the multilayer perceptron model with multitask heads (MLP-M). This model is trained using supervised learning to minimize error on the factual outcomes and consists of two multitask heads: one head for the treatment option which receives as input the dosage and estimates the dose-response curve and one head for the no-treatment option.

As can be seen in Table \ref{tab:mixing}, SCIGAN is capable of handling this setting and lends itself naturally to potentially mixed dosage and no-dosage treatment options.

\begin{table*}[h]
    	\centering
    	\begin{small}
    	\setlength\tabcolsep{2.2pt}
    	\begin{adjustbox}{max width=\textwidth}
    	\begin{tabular}{lccccccccc}
    		\toprule
    		\multirow{2}{*}{\textbf{Method}} & \multicolumn{3}{c}{\textbf{TCGA}} &
    		\multicolumn{3}{c}{\textbf{News}} & \multicolumn{3}{c}{\textbf{MIMIC}} \\
    		& \multicolumn{1}{c}{$\sqrt{\text{MISE}}$} &
    		\multicolumn{1}{c}{$\sqrt{\text{DPE}}$} &
    		\multicolumn{1}{c}{$\sqrt{\text{PE}}$} & \multicolumn{1}{c}{$\sqrt{\text{MISE}}$} &
    		\multicolumn{1}{c}{$\sqrt{\text{DPE}}$} &\multicolumn{1}{c}{$\sqrt{\text{PE}}$} & \multicolumn{1}{c}{$\sqrt{\text{MISE}}$} &
    		\multicolumn{1}{c}{$\sqrt{\text{DPE}}$} &\multicolumn{1}{c}{$\sqrt{\text{PE}}$} \\
    		\midrule
        	SCIGAN & $\textbf{1.28}\pm 0.09$ &  $\textbf{1.37}\pm 0.07$ & $\textbf{1.56}\pm 0.06$ & $\textbf{3.18} \pm 0.15$ & $\textbf{2.04}\pm 0.09$ & $\textbf{2.49} \pm 0.12$ &  $\textbf{0.61}\pm 0.08$ &  $\textbf{1.82} \pm 0.02$ &  $\textbf{1.89} \pm 0.03$ \\
    		\midrule
    		MLP-M & $2.08 \pm 0.12$ & $1.85\pm 0.16$ & $2.02 \pm 0.07$ & $4.68\pm 0.11$ & $2.45 \pm 0.08$  & $2.64 \pm 0.08$ & $1.56 \pm 0.08$ & $2.04 \pm 0.03$ & $2.14\pm 0.5$ \\
    		\bottomrule
    	\end{tabular}
    	\end{adjustbox}
    	\end{small}
    	\caption{Performance of individualized treatment-dose response when mixing treatment with no-treatment options.  Bold indicates the method with the best performance for each dataset.}
    	\vspace{-0.5cm}
    	\label{tab:mixing}
    \end{table*}

\subsection{Additional results on selection bias} \label{app:add_bias}
In Fig. \ref{fig:selection_bias_dpepe} we report the DPE for our treatment and dosage bias experiment from Section \ref{sec:bias} of the main paper.

\begin{figure}[H]
    \centering
    \subfloat[Treatment selection bias]{\includegraphics[height=4.5cm]{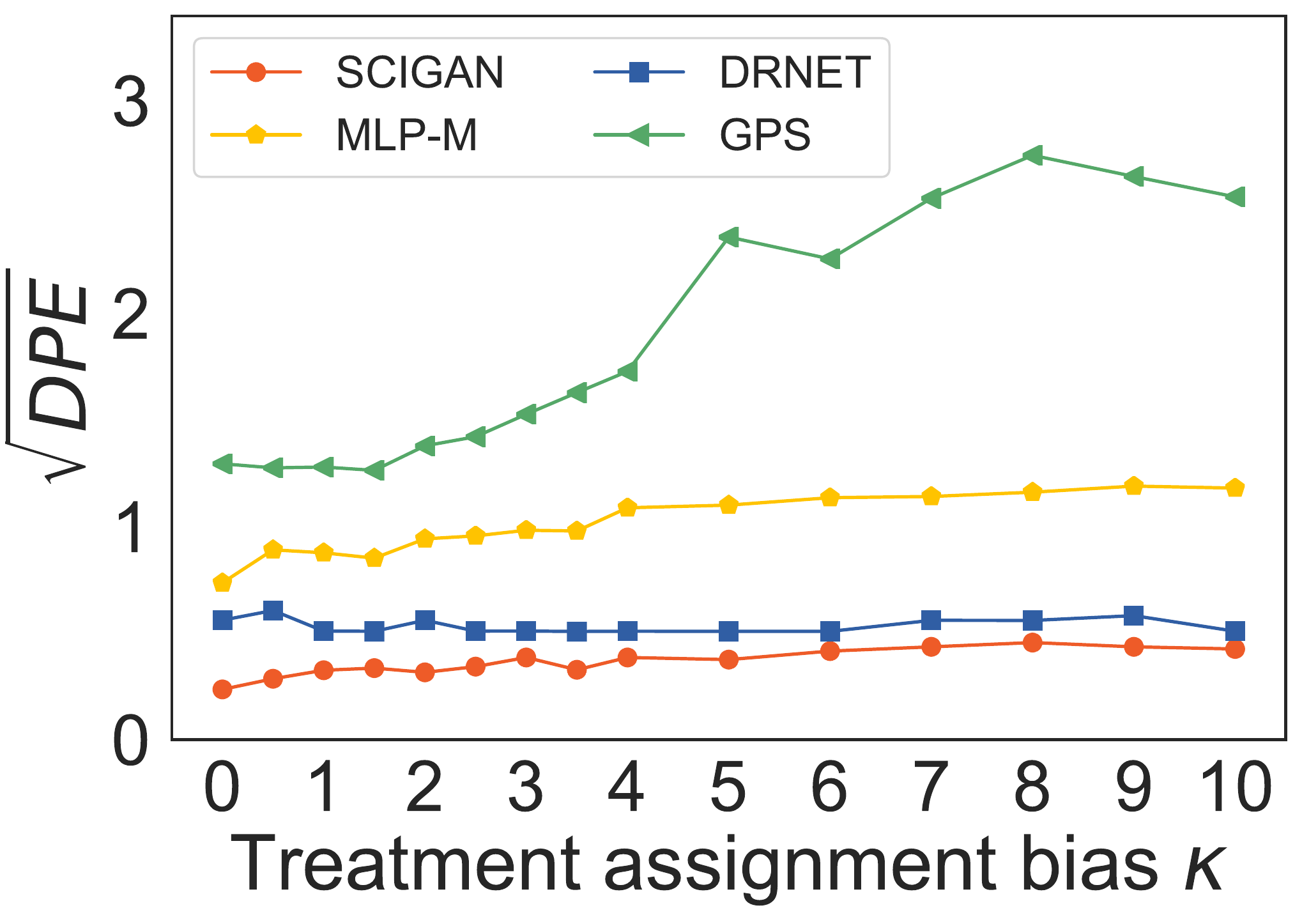} }
    \quad
    \subfloat[Dosage selection bias]{\includegraphics[height=4.5cm]{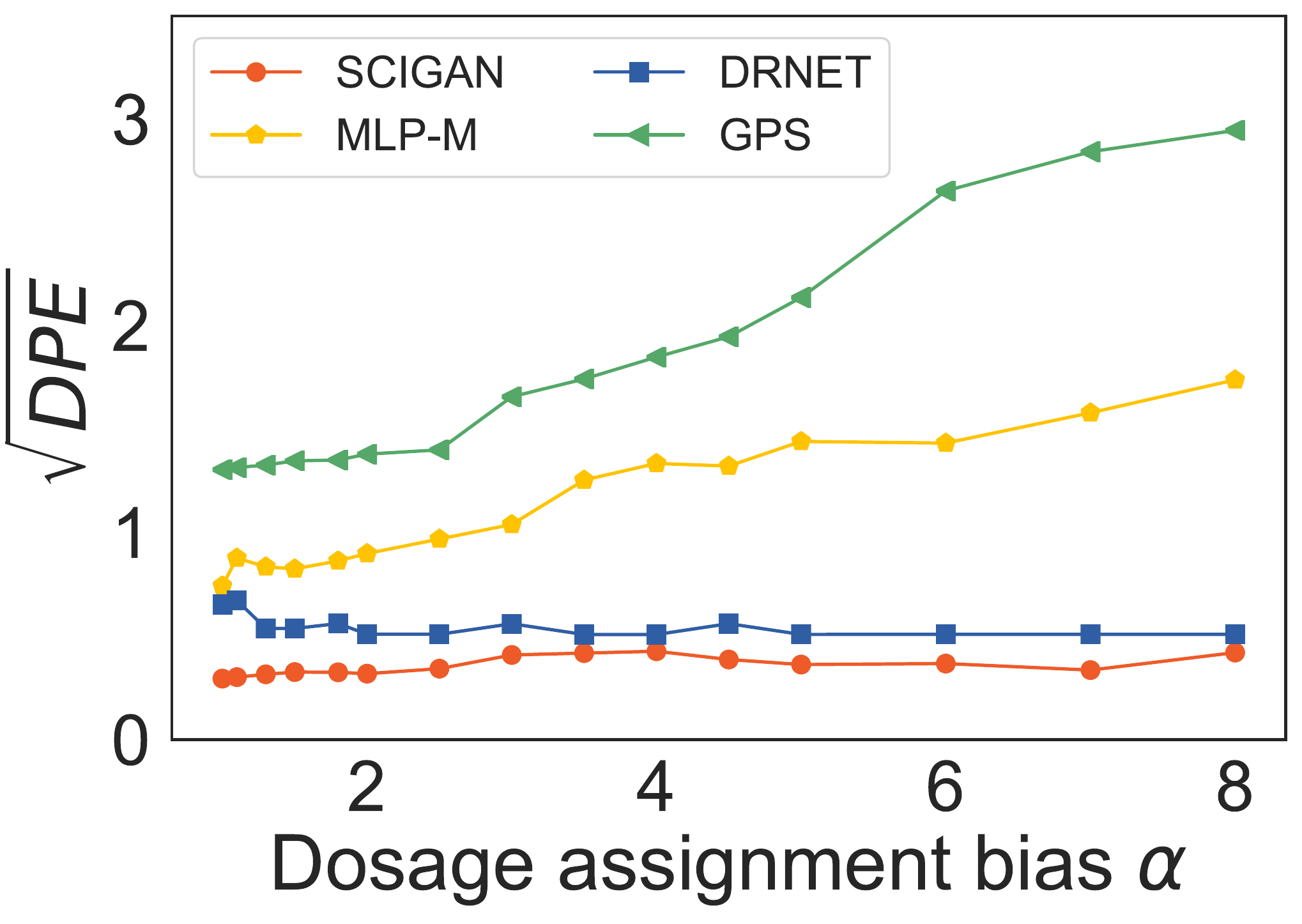}}
    \caption{Additional performance metrics of the 4 methods on datasets with varying bias levels on TCGA dataset}
    \label{fig:selection_bias_dpepe}
\end{figure}

\end{document}